\documentclass[a4paper,fleqn]{cas-dc}

\usepackage[numbers]{natbib}
\usepackage{amssymb}
\usepackage{amsmath}

\usepackage{color,array}
\usepackage{soul}
\usepackage{blindtext}
\usepackage{multicol}
\usepackage{algorithm}
\usepackage{algorithm}
\usepackage{amsthm}

\usepackage{svg}
\usepackage{amsmath,amssymb,amsfonts}
\usepackage{algpseudocode}
\usepackage{lipsum, color}
\usepackage{xcolor}
\usepackage{bm}
\usepackage{booktabs}

\usepackage{multirow}
\newtheorem{theorem}{Theorem} 
\newtheorem{corollary}{Corollary}
\newtheorem{definition}{Definition}
\newtheorem{remark}{Remark}
\newcommand{\doi}[1]{doi: \href{https://doi.org/#1}{\textcolor{blue!60!black}{#1}}}
\usepackage{graphicx}
\usepackage{epstopdf}
\usepackage[caption=true,font=small]{subfig}
\usepackage{stfloats}

\newtheorem{lemma}{Lemma}
\let\emptyset\varnothing

\def\tsc#1{\csdef{#1}{\textsc{\lowercase{#1}}\xspace}}
\tsc{WGM}
\tsc{QE}

\begin{document}
\let\WriteBookmarks\relax
\def\floatpagepagefraction{1}
\def\textpagefraction{.001}
\shorttitle{Upper Bounds on the Generalization Error of Deep Learning Models}
\shortauthors{A.R Nuhu et~al.}

\title [mode = title]{Upper Bounds on the Generalization Error of Deep Learning Models via Local Robustness and Stability}                      


\author[label1]{Abdul-Rauf Nuhu}
\author[label1]{Parham M. Kebria}
\author[label1]{Vahid Hemmati}
\author[label2]{Mahmoud Nabil Mahmoud}
\author[label3]{Edward Tunstel} 
\author[label1]{and Abdollah Homaifar\corref{cor1}}


\affiliation[label1]{organization={North Carolina Agricultural and Technical State University},
             Department = {Dept. Electrical and Computer Engineering},
            addressline={1601 E Market St.}, 
            city={Greensboro},
            postcode={27411}, 
            state={NC},
            country={USA}}

\affiliation[label2]{organization={University of Alabama},
            addressline={Dept. Computer Science}, 
            city={Tuscaloosa},
            postcode={35487}, 
            state={AL},
            country={USA}}
\affiliation[label3]{organization={Southwest Research Institute},
            city={San Antonio},
            postcode={78238}, 
            state={TX},
            country={USA}}
\cortext[cor1]{Corresponding author: homaifar@ncat.edu}

\begin{abstract}
Generalization is a critical property of data-driven models, particularly deep learning models deployed in safety-critical applications. Robustness-based generalization bounds have gained attention as a principled way to link robustness properties to generalization performance, often in a data-dependent manner. However, most existing bounds suffer from vacuousness in practical settings, yielding loose upper bounds that greatly exceed the actual error rates and limiting their usefulness for real-world evaluation. While this issue is often attributed to the uncertainty term, a substantial part of the problem originates from the robustness term itself, particularly for the 0-1 loss. Existing approaches typically treat the robustness term as a global measure, ignoring its variation across different sub-regions of the input space. In this work, we propose a  generalization bound that addresses this limitation by scaling the robustness term according to the number of stable and unstable samples within each sub-region. Our bounds incorporate both data- and model-dependent factors while maintaining practical relevance (yielding tighter upper bounds on true error). Experiments on models trained on the ImageNet dataset show that our bounds remain consistently non-vacuous and achieve the tightest estimates among existing methods, closely aligning with empirical performance across a range of robust deep neural networks.
\end{abstract}


\maketitle

\section{Introduction}\label{sec1}
Data-driven models, particularly deep learning classifiers, have recently achieved remarkable performance across a wide range of applications. These models have become integral to many safety-critical systems \cite{10747123, 9659972, kebria2019evaluating, zeleke2025integrated, Abdul2022NSA}. However, they remain vulnerable to common corruptions such as blur, noise, and adversarial examples \cite{GUO2023109308, 10316046, Ding_Neurocomputing_2025}. Robustness enhancement methods, such as robust optimization \cite{ben1998robust, GABREL2014471}, have emerged as influential tools for addressing data uncertainty \cite{Ding_Neurocomputing_2025, kebria2019robust}. These approaches leverage concepts from convexity and duality to derive feasible solutions for optimization problems. They have been successfully applied across various domains, including machine learning, to improve the robustness of deep neural networks \cite{nuhu2025validationstrategydeeplearning, kebria2019deep, NEURIPS2021_ea4c796c}.

Inspired by robust optimization, the authors in \cite{xu2012robustness} demonstrated that robust algorithms generalize effectively to unseen data across different models, including deep learning architectures. This establishes robustness as an alternative perspective for studying generalization \cite{ROHLFS_Neurocomputing_2025, 8919696, kawaguchi2022generalization, LIU_Neurocomputing_2025, BALLESTER_Neurocomputing_2025}. According to Definition  \eqref{def:robustness}, a learning algorithm $\mathcal{A}$ is considered robust if the loss $\ell$ of its output hypothesis $\mathcal{A_S}$ (a model returned by a learning algorithm $\mathcal{A}$ when trained on training dataset $\mathcal{S}$) remains similar for nearby input samples. Building on this notion, prior work has shown that the generalization error of $\mathcal{A}_{\mathcal{S}}$ can be upper bounded by two key components. The first is a robustness term, $\epsilon(\mathcal{S})$ and the second is an uncertainty term, $\sqrt{K/n}$. Where, $K$ represents the number of disjoint clusters that partition the input space, and $n$ is the training sample size. More recently, authors in \cite{kawaguchi2022robustness} refined this uncertainty term, thereby improving the practicality of robustness-based generalization bounds.

Based on these developments, the robustness-based generalization framework has attracted significant attention in the machine learning community \cite{ Ding_Neurocomputing_2025, kawaguchi2022robustness, 7934087, kebria2018deep}. This interest stems from the intuitive nature of the framework and the theoretically grounded definition of the robustness term $\epsilon(\mathcal{S})$ \cite{kawaguchi2022robustness}. Nevertheless, under the 0–1 loss, even a single misclassified sample can result in $\epsilon(\mathcal{S}) = 1$, reducing the bound to that of the worst-case model and rendering it vacuous \cite{than2025gentle}. Although the supremum operator in the definition is designed to capture worst-case behavior, several theoretical studies impose restrictions to keep $\epsilon(\mathcal{S}) < 1$ \cite{BELLET2015259, 7811216}. However, such constraints are difficult to justify in practice, as even high-performing models may misclassify certain inputs. To address this issue, recent work proposes instance-aware modifications to the robustness definition to reduce the influence of outlier samples that disproportionately inflate $\epsilon(\mathcal{S})$ \cite{than2025gentle}. This adjustment retains worst-case sensitivity while producing more meaningful generalization guarantees.
\begin{figure*}
\centering
\subfloat[Original training data instances classified correctly as the true label "goldfish."]{\includegraphics[width=0.475\linewidth]{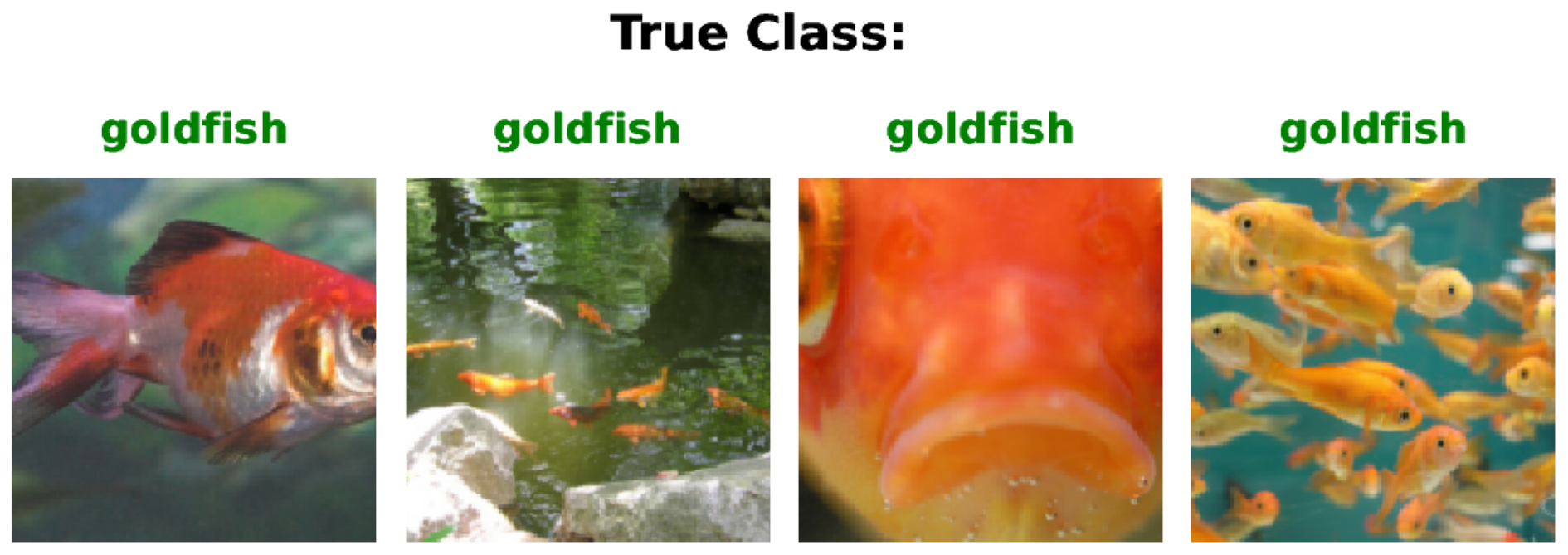}%
\label{fig:weakRobust-a}}
\hfil
\subfloat[A perturbed version of the data instances in (a) misclassified to belong to different classes.]{\includegraphics[width=0.475\linewidth]{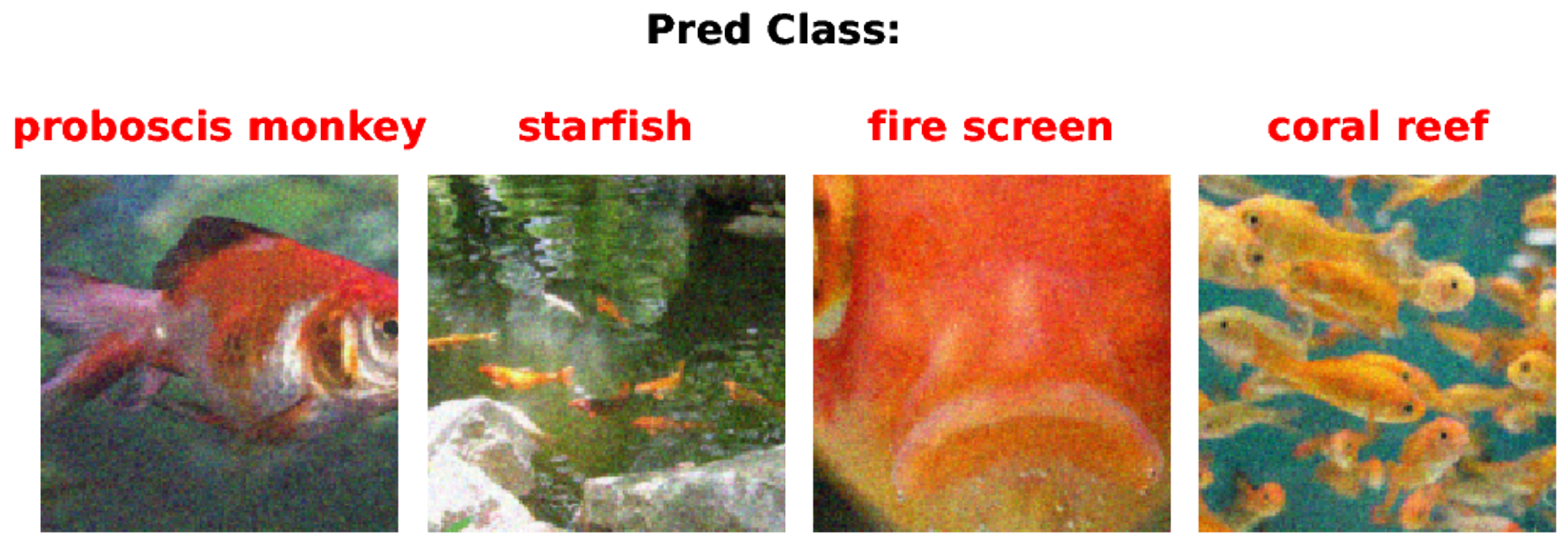}%
\label{fig:weakRobust-b}}
\caption{A trained AllConvNet model classifies the original instances (a) correctly as the truth class, goldfish, while their perturbed variants are misclassified. Data instances with the majority of their neighbors misclassified are considered weak robust samples.}
\label{fig-weakRobust-Samples}
\end{figure*}

Despite these advances, robustness-based generalization bounds remain limited when applied to specific trained models, especially in tasks such as model selection. Comparing models trained with different or stochastic algorithms remains challenging \cite{xu2012robustness}, highlighting the need for more model-sensitive generalization measures. In this context, authors in \cite{than2025gentle} introduce a localized robustness formulation that decomposes the global robustness term into cluster-specific components $\epsilon_j(h)$. Each $\epsilon_j(h)$ quantifies the maximum change in loss between two similar samples within the same local neighborhood (e.g., cluster $j$ and model $h$). This formulation avoids the unrealistic assumption of uniform robustness across the dataset and instead captures heterogeneous local behaviors. However, by scaling each cluster’s robustness contribution by the number of samples within it, their method implicitly assumes uniform influence from all instances in a cluster. In practice, stable samples in high-variance clusters can overshadow the effect of vulnerable ones, resulting in misleading robustness estimates. To address this, we refine the framework by explicitly identifying stable and unstable samples. We assign robustness contributions that reflect the stable and unstable samples behaviors. This adjustment ensures that the influence of highly variable or error-prone samples is appropriately captured in the generalization bound.

To formalize this idea, we define a sample as stable if the model classifies the original input and its perturbed counterparts correctly. A sample is labeled unstable when the misclassification rate among its perturbed neighbors exceeds a predefined threshold \cite{nuhu2025validationstrategydeeplearning}. As illustrated in Figure~\ref{fig-weakRobust-Samples}, the clean images in (a) are correctly labeled as "goldfish", while their perturbed versions in (b), though invisible, are misclassified. This exemplifies unstable behavior. These unstable samples often exhibit high loss variability and, under the robustness definition in \eqref{def:robustness}, can cause $\epsilon(\mathcal{S})$ to approach 1, producing vacuous generalization bounds. 
Figure~\ref{fig:samples_stability} illustrates that stable samples maintain relatively constant loss values across iterations, whereas unstable samples exhibit significant volatility. 
\begin{figure}
    \centering
    \includegraphics[width=0.85\linewidth]{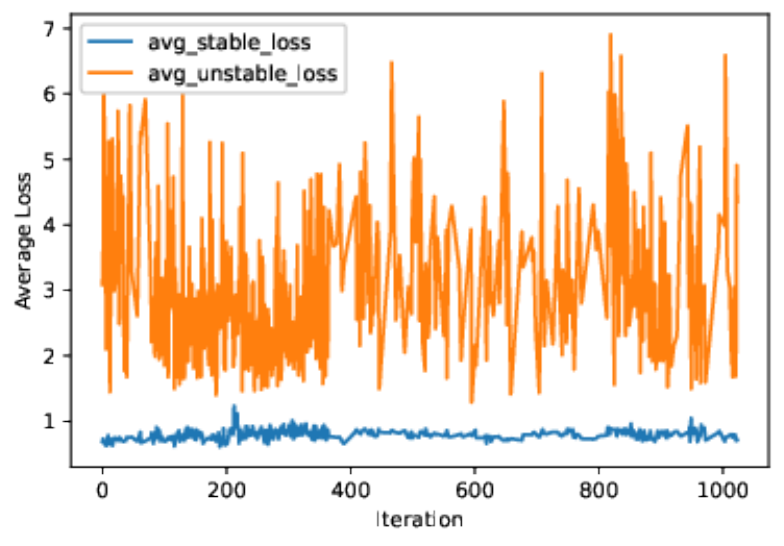}
    \caption{Evolution of the average unnormalized cross-entropy loss for stable and unstable samples on the ImageNet dataset using a ResNet-50 model. Unstable samples consistently exhibit higher and more variable losses across iterations, whereas stable samples maintain low, consistent losses. This distinct behavior highlights the effectiveness of unstable samples as proxies for assessing generalization through robustness-based bounds.}
    \label{fig:samples_stability}
\end{figure}
Using a global average robustness term for the entire dataset can therefore overestimate the generalization bound. The same issue arises when local cluster robustness is scaled without distinguishing between stable and unstable samples. This overestimation occurs because unstable samples contribute the most to the generalization error \cite{7934087}.
When their effect is combined with stable samples, it becomes masked. This masking leads to an overestimation of the generalization bound. As a result, the bound becomes inflated and does not reflect the true behavior of a model. Our framework addresses this limitation by explicitly isolating unstable samples and systematically accounting for their influence. By disentangling their impact from that of stable samples, it produces tighter and more reliable robustness-based generalization guarantees. To best of our knowledge, we are the first to use samples stability to provide empirical results on generalization guarantees measure.

Our contributions in this paper are as follows: 
\begin{itemize}
    \item Motivated by the observation in \cite{than2025gentle} that robustness-based generalization bounds remain vacuous for classification problems with overlapping classes, wwe provide evidence that this limitation is linked to structural aspects of the formulation. Unlike prior approaches that uniformly treat all samples across clusters, our framework differentiates between stable and unstable regions, yielding significantly tighter and more meaningful generalization estimates in practice.
    \item We propose a new class of model-dependent generalization bounds based on local model behavior across the input space. Our method separates stable and unstable clusters, allowing robustness contributions to reflect actual variations in the data. It does not require robustness assumptions about the model and yields bounds that are provably tighter than existing robustness-based approaches.
    \item We empirically evaluate our proposed bounds on real-world datasets using modern neural network architectures and show that they correlate more strongly with actual model performance than existing baseline bounds. This improved alignment makes them a more reliable tool for model selection and comparative performance evaluation.
\end{itemize}

The rest of the paper is structured as follows. Section~\ref{section:2} presents the key definitions and existing bounds that form the foundation for our proposed bounds. Section~\ref{section:3} presents the proposed bounds along with their implementation details.  Section~\ref{section:4} summarizes the proposed framework for our theorems implementation. Section~\ref{section:5} describes the experimental setup and compares our results with existing bounds. Finally, Section~\ref{section:6} concludes the paper and outlines directions for future research. Proofs of our bounds are provided in~\ref{appendx:A}, while additional results are presented in~\ref{Appendix_B}.
\begin{figure}
\begin{center}
\includegraphics[width=0.85\linewidth]{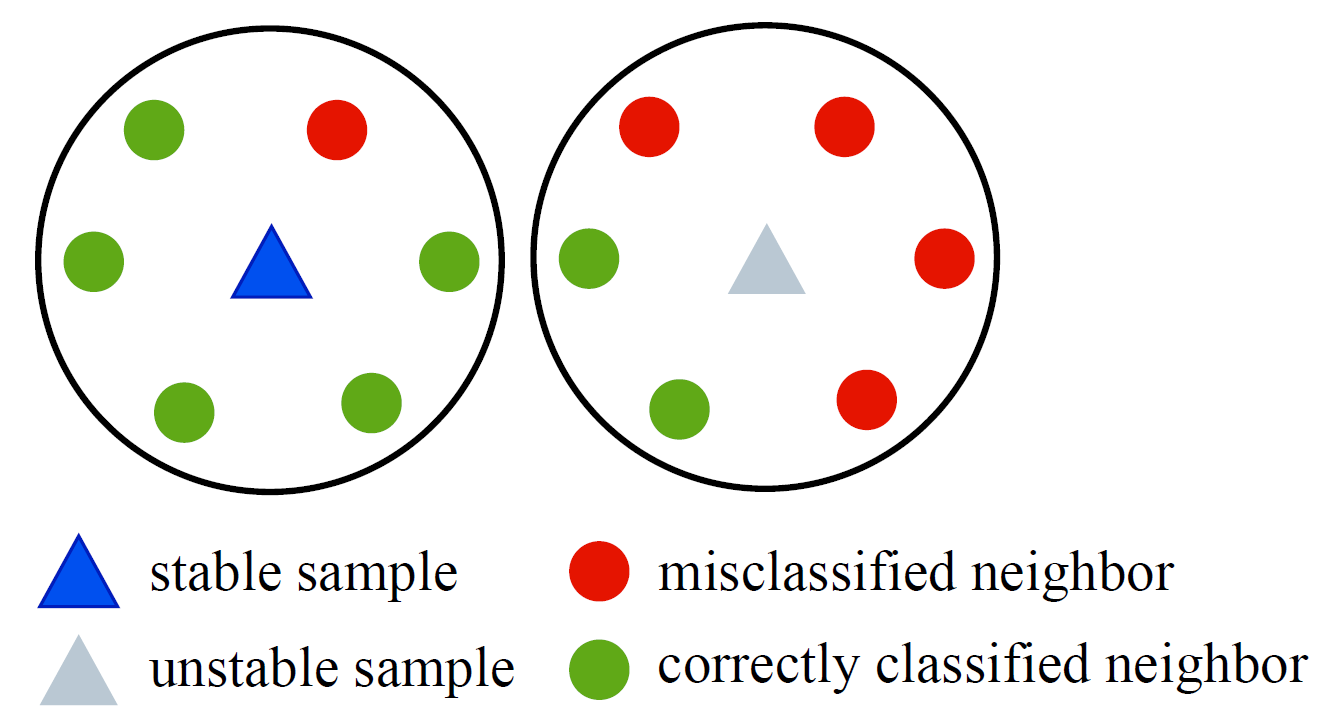}
\caption{Data instances and the corresponding neighbors. The blue and gray triangles show the original instances; the green and red circles are the neighbors generated from perturbing an original instance.}
\label{fig-ngbs}
\end{center}
\end{figure}

\section{Preliminaries}\label{section:2}
This section introduces key terminologies and notations. Then, it covers  foundational definitions and theorems on robustness-based generalization bounds. Finally, it concludes with factors behind vacuousness of these bounds and prepares the ground for our proposed theorems. 

\subsection{Terminology Definitions} We discuss the key terminologies used.\\

\textbf{Neighbors:} \textit{Samples obtained by perturbing original instances from a given dataset.} As illustrated in Figure~\ref{fig-ngbs}, the blue and gray triangles denote original samples, while the green and red circles represent their corresponding neighbors, generated through random perturbations. Neighbors may also be generated using natural variations or adversarial perturbation techniques.\\

\textbf{Unstable sample:} A data instance is classified as an unstable sample if its local neighborhood exhibits high sensitivity under perturbations, i.e., a significant proportion of its perturbed neighbors are misclassified by the model. As illustrated in Figure~\ref{fig-weakRobust-Samples}, the original inputs in (a) are correctly classified, while their perturbed counterparts in (b) are frequently misclassified, indicating weak robustness.

Formally, this behavior reflects regions where the loss function varies sharply under small input perturbations. Such samples are closely associated with small classification margins or brittle decision boundaries, as established in prior works on margin-based stability \cite{xu2012robustness, NEURIPS2024_29753d93, nuhu2025validationstrategydeeplearning}. These works show that samples with low margins contribute disproportionately to instability and robustness degradation. Figure~\ref{fig:samples_stability} further illustrates this distinction: stable samples maintain relatively constant loss values across iterations, whereas unstable samples exhibit significant volatility.

\subsection{Notations} 
Let $\mathcal{Z}$ represents the sample space, and $\mathcal{S} \subset \mathcal{Z}$ denotes a training set consisting of $n$ samples $(z_1, \dots, z_n)$, with $n=|\mathcal{S}|$ and $Z_m \subset \mathcal{Z}$ represents an arbitrary test dataset ($Z_m \cap \mathcal{S} = \emptyset $) drawn from $\mathcal{Z}$. $z_i$ denotes an arbitrary sample drawn from $\mathcal{S}$ while $z$ denotes an arbitrary sample drawn from $\mathcal{Z}$ i.e $z \in \mathcal{Z} \setminus \mathcal{S}$.
We define $\hat{\mathcal{S}}$ as the stable samples, and $\widetilde{\mathcal{S}} = \mathcal{S} \setminus \hat{\mathcal{S}}$ as the unstable samples in the training set $\mathcal{S}$. The size of $\hat{\mathcal{S}}$ is denoted as $\hat{n}(\mathcal{A_S})$. Similarly, the size of $\widetilde{\mathcal{S}}$ is denoted as $ \widetilde{n}(\mathcal{A_S})$.  For a given $K \in \mathbb{N}$ with $K \geq 1$, we denote $[K] = \{1, \dots, K\}$ as the number of clusters. 

Let $\beta(\mathcal{Z}):= \bigcup_{k=1}^{K} C_k$ be a partition of $\mathcal{Z}$ into $K$ disjoint nonempty subsets. We assume that the input space $\mathcal{Z} \subseteq \mathbb{R}^d$ is a metric space, where local neighborhoods can be defined via small perturbations under a chosen norm. The notation $||\cdot||_p$ refers to the standard $p$-norm of a vector.

Consider a learning problem defined by a hypothesis class $\mathcal{H}$ and a loss function $\ell : \mathcal{H} \times \mathcal{Z} \to \mathbb{R}$. Given a distribution $\mathbb{P}$ over $\mathcal{Z}$, the generalization capability of $\mathcal{A}_{\mathcal{S}} \in \mathcal{H}$, a model returned by learning algorithm $\mathcal{A}$ on training set $\mathcal{S}$, is measured by its expected risk 
\[
\mathcal{R}_Z(\mathcal{A_S}) = \mathbb{E}_{z \sim \mathbb{P}}[\ell(\mathcal{A_S}, z)].
\]  
Since this expectation cannot be computed directly, the empirical risk  
\[
\hat{\mathcal{R}}_{\mathcal{S}}(\mathcal{A_S}) = \frac{1}{n} \sum_{i=1}^{n} \ell(\mathcal{A_S}, s_i)
\]  
is typically used to approximate it \cite{7934087, CHENG_Neurocomputing_2025}. The expected risk is then upper bounded by the empirical risk plus additional terms, as given:  

\[
    \mathcal{R}_Z(\mathcal{A_S})
    \;\leq\;
    \hat{\mathcal{R}}_{\mathcal{S}}(\mathcal{A_S})
    + \mathcal{O}\!\left(\sqrt{\tfrac{C}{n}}\right)
\]
where C denotes the complexity of a model.

\subsection{Robustness-based generalization bounds}
Previous studies in \cite{xu2012robustness, kawaguchi2022robustness} used algorithmic robustness to bound the expected risk $\mathcal{R}_Z(\mathcal{A_S})$.\\

\begin{definition}[Algorithmic Robustness] \label{def:robustness}
A learning algorithm \( \mathcal{A} \) is said to be \((K, \epsilon(\cdot))\)-robust, for some \( K \in \mathbb{N} \) and a function \( \epsilon : \mathcal{Z}^n \to \mathbb{R} \), if the sample space \( \mathcal{Z} \) can be partitioned into \( K \) disjoint subsets \( \{ C_k \}_{k=1}^K \) such that the following holds:

For all training samples \( \mathcal{S} \in \mathcal{Z}^n \), and \(  \forall s \in \mathcal{S} \),  \( \forall z \in \mathcal{Z} \), and \( k \in [K] \),  
if \( s, z \in C_k \), then  
\[
\left| \ell(\mathcal{A}_\mathcal{S}, s) - \ell(\mathcal{A}_\mathcal{S}, z) \right| \leq \epsilon(\mathcal{S}).
\]
\end{definition}

Intuitively, algorithmic robustness means that the learned algorithm $\mathcal{A}_S$ behaves consistently for samples that belong to the same partition of the sample space as the training data, as measured by the loss function $\ell$ \cite{kawaguchi2022robustness}. In other words, samples that are similar to the training data incur comparable loss values. This provides intuition for why robustness can support generalization to nearby samples. To formalize the relation between robustness and generalization of a trained model, studies in \cite{xu2012robustness, kawaguchi2022robustness} incorporate the robustness properties of a learning algorithm in \eqref{def:robustness} and proposed the following theorems.
\begin{theorem}[Xu \& Mannor \cite{xu2012robustness}] \label{thm:xu-generalization}
Let $\ell(h, z)$ be a non-negative loss function upper bounded by $M$, i.e., $\ell(h, z) \leq M$ for all $h \in \mathcal{H}$ and $z \in \mathcal{Z}$. If the learning algorithm $\mathcal{A}$ is $(K, \epsilon(\cdot))$-robust, then for any $\delta > 0$, 
\[u_{1}(K, \mathcal{S}, \delta) = M\sqrt{\frac{2K \ln 2 + 2\ln(1/\delta)}{n}}\]
with probability of at least $1 - \delta$, and $n$ samples $S = \{z_1, \dots, z_n\}$ drawn from an iid (independent and identically distributed), the following holds:
\begin{equation}
\begin{split}
   \mathcal{R}_Z(\mathcal{A_S}) \leq \hat{\mathcal{R}}_{\mathcal{S}}(\mathcal{A_S}) + \epsilon(\mathcal{S})
   + u_{1}(K, \mathcal{S}, \delta) 
\end{split}
\label{eqn_2r} 
\end{equation}
\end{theorem}

It can be seen from Theorem~\ref{thm:xu-generalization} that the expected risk $\mathcal{R}_Z(\mathcal{A_S})$ is bounded by the empirical risk $\hat{\mathcal{R}}_{\mathcal{S}}(\mathcal{A_S})$ and some additional terms. The first term of these additional terms capture the robustness properties of the learned model $\mathcal{A_S}$. The last term of the additional terms in bound \eqref{eqn_2r} scales with $\sqrt{K}$ and has been the main disadvantage of this bound \cite{kawaguchi2022robustness}. Through the concentration of multinomial random variables, authors in \cite{kawaguchi2022robustness} proposed a theorem that reduces the $K$ dependence in Theorem~\ref{thm:xu-generalization}. The theorem is as follows:\\

\begin{theorem}[{Kawaguchi et al., \cite{kawaguchi2022robustness}}] \label{thm:kawaguchi}
Following the same assumptions and notations as in Theorem~\ref{thm:xu-generalization}, for any\\
$\delta > 0$, denote
\begin{align*}
u_{2}(K, \mathcal{S}, \delta) 
&= \phi(\mathcal{A_S}) \Bigg(
      (\sqrt{2}+1)\sqrt{\frac{|T_S| \ln(2K/\delta)}{n}}\\ 
      &+ \frac{2|T_S| \ln(2K/\delta)}{n} 
   \Bigg)
\end{align*}

with probability at least $1 - \delta$, and $n$ samples $S = \{z_1, \dots, z_n\}$ drawn from an iid, the following generalization bound holds:
\begin{equation}
\begin{split}
    \mathcal{R}_Z(\mathcal{A_S}) \leq \hat{\mathcal{R}}_{\mathcal{S}}(\mathcal{A_S}) + \epsilon(\mathcal{S}) 
    + u_{2}(K, \mathcal{S}, \delta)
\end{split}
\label{eqn_2r_1}
\end{equation}
\end{theorem}

\begin{align*}
\text{where } \phi(A_S) &:= \max_{z \in \mathcal{Z}} \left\{\ell(\mathcal{A_S}, z) \right\}, \\
T_S &:= \left\{ k \in [K] \;\middle|\; \left| \mathcal{I}_k^S \right| \geq 1 \right\},\\
\mathcal{I}_k^S &:= \{ i \in [n] : z_i \in C_k \} 
\end{align*}
Theorem \ref{thm:kawaguchi} significantly improves over the previous bound \eqref{eqn_2r}. It has a far less dependency on $K$, as $\sqrt{K}$ is reduced to $\sqrt{\ln K}$. $\phi(\mathcal{A_S})$ replaces the maximum over the entire model family (hypothesis space) with a single hypothesis returned by the algorithm ($\mathcal{A_S}$). Compared with $u_{1}(K, \mathcal{S}, \delta)$, the uncertainty term $u_{2}(K, \mathcal{S}, \delta)$ becomes significantly smaller. Although this is a significant improvement, authors in \cite{kawaguchi2022robustness} further showed that $u_{2}(K, \mathcal{S}, \delta)$ can be further improved with 

\[u_{3}(K, \mathcal{S}, \delta) = \mathcal{Q}_1\sqrt{\frac{\ln{(2K/\delta)}}{n}} + \frac{2\mathcal{Q}_2\ln{(2K/\delta)}}{n}\]\\ 
where 
\begin{align*}
\mathcal{Q}_1 &:= \sum_{k \in T_S} \left( \alpha_{\mathcal{T}_{S}^c}(A_S) + \sqrt{2} \, \alpha_k(A_S) \right) \sqrt{\frac{|\mathcal{I}_k^S|}{n}}, \\
\mathcal{Q}_2 &:= \alpha_{\mathcal{T}_{S}^c}(A_S) \cdot |T_S| + \sum_{k \in T_S} \alpha_k(A_S),
\end{align*}

\begin{align*}
T_S &:= \left\{ k \in [K] : |\mathcal{I}_k^S| \geq 1 \right\}, \\
\mathcal{I}_k^S &:= \left\{ i \in [n] : z_i \in C_k \right\}, \\
\alpha_k(h) &:= \mathbb{E}_{z} \left[ \ell(h, z) \mid z \in C_k \right], \\
\alpha_{\mathcal{T}_{S}^c}(A_S) &:= \max_{k \in \mathcal{T}_{S}^c} \alpha_k(A_S), \\
\mathcal{T}_{S}^c &:= [K] \setminus T_S.
\end{align*}

Although the uncertainty term $u_{1}(K, \mathcal{S}, \delta)$ has improved to $u_{3}(K, \mathcal{S}, \delta)$, the robustness term $\epsilon(\mathcal{S})$ remains the same across different proposed bounds \cite{than2025gentle}. As a result, there are still some issues with these bounds, and we discuss this matter in the following.
\subsection{Vacuousness and its main causes}
The robustness-based generalization bounds can become vacuous, primarily due to the conservative nature of the robustness condition defined in Definition~\ref{def:robustness}. This issue was first identified by the authors in \cite{than2025gentle}. Consider a model $\mathcal{A_S} \in \mathcal{H}$ returned by a learning algorithm $\mathcal{A}$. Definition~\ref{def:robustness} implies that 
\[\epsilon(\mathcal{S}) \geq \max_{j \in T_S} \epsilon_j(\mathcal{A_S})\]  where 
\[\epsilon_j(\mathcal{A_S}) = \max_{s, z \in C_j}|\ell(\mathcal{A_S}, z)-\ell(\mathcal{A_S}, s)|\] 
Even for highly accurate models, bounds \eqref{eqn_2r} and \eqref{eqn_2r_1} can still become vacuous. In practice, it is possible to have  incorrect predictions from a highly accurate model. Therefore, for losses normalized to [0,1], an incorrect prediction can result in $\epsilon(\mathcal{S})=1$ and thus making the bounds vacuous. In a nutshell, Definition~\ref{def:robustness} requires the following specific operations: 
\begin{itemize}
    \item Supremum operation: The robustness term $\epsilon(\mathcal{S})$ is computed by taking the supremum of local robustness levels ($\epsilon_j$) across all clusters of the input space that contains some training examples. This means that even a single unstable region, where the local robustness is poor (e.g., $\epsilon_j=1$ for $0-1$ loss, indicating a region where misclassification is possible). This can inflate the global robustness measure and in practice, it is common for models to make at least a few mistakes. So, taking the supremum of local robustness levels will almost always result in vacuous bounds.
    \item Stochasticity Inclusion Challenge: In practice, most learning algorithms $\mathcal{A}$ are stochastic, meaning that different runs, given the same training set $\mathcal{S}$ and hyperparameters, can yield different trained models. Consequently, evaluating the robustness metric $\epsilon(\mathcal{S})$ must account for this randomness. A more realistic formulation is:
    \[
    \epsilon(\mathcal{S}) \geq \sup_{\zeta, j} \epsilon_j(\mathcal{A}_{\mathcal{S}}(\zeta)),
    \]
    where $\zeta$ denotes the source of stochasticity and $\epsilon_j$ measures the robustness of the $j$th cluster. Fully characterizing $\epsilon(\mathcal{S})$ would therefore require considering all stochastic variations of $\mathcal{A}$, which is computationally infeasible in most settings. Moreover, some stochastic outcomes can produce models with substantially reduced robustness. In such cases, assuming a single global robustness term $\epsilon(\mathcal{S})$ for the entire dataset, as in existing approaches, can render the bounds vacuous \cite{than2025gentle}. These challenges highlight the need for a more tractable and reliable approach to evaluate robustness under algorithmic stochasticity.
\end{itemize}
To mitigate the vacuousness issues that stem from supremum operation and stochasticity inclusion, the authors in \cite{than2025gentle} proposed the following theorem:
\begin{theorem}[{Khoat et al., \cite{than2025gentle}} Theorem 4] \label{thm:khoat}
Consider a model $\mathbf{h} \in \mathcal{H}$ learned from $\mathcal{S}$ and a bounded loss $\ell$. For each $j \in [K]$, let $\epsilon_j(h) = \max_{s, z \in C_j}|\ell (h, z)-\ell (h, s)|$.  For any $\delta > 0$, with probability at least $1 - \delta$ of $n$ samples $S = \{z_1, \dots, z_n\}$ drawn from an iid, the following generalization bound holds:
\begin{equation}
\begin{split}
    \mathcal{R}_Z(\mathbf{h}) \leq \hat{\mathcal{R}}_{\mathcal{S}}(\mathbf{h}) + \sum_{j \in T_S} \frac{n_j}{n} \epsilon_j(\mathbf{h}) + u_{3}(K, \mathcal{S}, \delta)\\
\end{split}
\label{eqn_2r_2}
\end{equation}
\end{theorem}
This theorem shows that the expected loss of a model can be bounded by $\epsilon_j(\mathbf{h})$. $\epsilon_j(\mathbf{h})$ describes the local robustness of $\mathbf{h}$ at different regions and suggests that a model can generalize well when it is locally robust in those regions. This highlights that a model can have a small expected loss over the whole sample space if it has a small training loss and is locally robust. The global robustness term $\epsilon(\mathcal{S})$ being replaced by a finer quantity $\sum_{j \in T_S} \frac{n_j}{n} \epsilon_j(\mathbf{h})$, capturing the ``stochasticity inclusion'' in bound \eqref{eqn_2r} and \eqref{eqn_2r_1}. Authors in \cite{than2025gentle} provided an improved version of Theorem \ref{thm:khoat} (bound \ref{thm:khoat_6}) by considering the averages of loss at the clusters (local regions).

\begin{theorem}[{Khoat et al., \cite{than2025gentle}} Theorem 5] \label{thm:khoat_6}
Given the notions in Theorem \ref{thm:khoat} For any $\delta > 0$, with probability at least $1 - \delta$ over an iid draw of $n$ samples $S = \{z_1, \dots, z_n\}$, the following generalization bound holds:
\begin{equation}
\begin{split}
    \mathcal{R}_Z(\mathbf{h}) \leq \hat{\mathcal{R}}_{\mathcal{S}}(\mathbf{h}) + \sum_{j \in T_S} \frac{n_j}{n} \bar{\epsilon}_j(\mathbf{h}) + u_{3}(K, \mathcal{S}, \delta)\\
\end{split}
\label{eqn_2r_3}
\end{equation}
where:
\[\bar{\epsilon}_j = \frac{1}{|n_j|} \sum_{s \in C_j} \mathbb{E}_{z \in Z_j} \big[\,|\ell(\mathbf{h}, z) - \ell(\mathbf{h}, s)|\,\big] \]

\end{theorem}

\section{Local Stability Behaviors and Generalization}\label{section:3}
In this section, we discuss our novel bounds that connects local behaviors (per-input resilience) with the generalization ability of a given model. Our bounds are both model-specific and data-dependent, and they relax the strict robustness assumptions imposed by previous approaches \cite{xu2012robustness, kawaguchi2022robustness, than2025gentle}.
\subsection{Cluster Stability and Per-Input Robustness based Bounds} In the previous section, the algorithmic robustness-based generalization bounds quantify model robustness across the entire data space through the term $\epsilon(\mathcal{S})$. This global measure is derived using a supremum \cite{than2025gentle}, effectively capturing the worst-case local robustness of the model within individual regions or clusters of the input space. This robustness term is then assumed as the average (global) robustness term for the entire dataset \cite{xu2012robustness, kawaguchi2022robustness}. Consequently, this formulation often leads to vacuous generalization bound \cite{than2025gentle}. Since the supremum-based measure $\epsilon(\mathcal{S})$ is dominated by worst-case local variations rather than representing the entire dataset, the robustness term should instead be scaled in proportion to the fraction of a cluster samples.

We address the limitations of prior bounds by incorporating local robustness across the clusters into the generalization bound. Specifically, we decompose the supremum into two components: one from the stable subset and the other from the unstable subset. Unlike Theorems~\ref{thm:xu-generalization}, \ref{thm:kawaguchi}, \ref{thm:khoat}, and \ref{thm:khoat_6}, our approach does not apply the global supremum value $\epsilon(\mathcal{S})$ uniformly to all samples within a cluster. Instead, the supremum is separated into the stable and unstable parts and scaled according to their respective sample proportions. This results in a bound that is less conservative and more sensitive to the actual distribution of instability in the data. 
Our approach is computationally practical, as stable and unstable instances can be identified by measuring local loss differences or margin violations. It is based on the observation that unseen instances behave like similar training samples. When an unseen instance is close to stable samples, the model is likely to show low loss. Conversely, proximity to unstable samples leads the model to exhibit higher variations in loss \cite{7934087}. By capturing this behavior, the bound reflects the model’s performance more accurately and supports better comparison and selection. The following theorems present our bounds.
\begin{theorem}[Cluster \textbf{Stability}-Based Robustness Decomposition] \label{thm:cluster_bound}
Let $\mathcal{A_S}$ be a model trained on a dataset $\mathcal{S}$ using a bounded loss function $\ell$. Assume the sample space $\mathcal{Z}$ can be partitioned into stable $\hat{\mathcal{Z}}$ and unstable subsets $\widetilde{\mathcal{Z}}$, where $\mathcal{Z} = \hat{\mathcal{Z}} \cup \widetilde{\mathcal{Z}}$.

We define clustering independently over each subset:
\begin{itemize}
   \item Let $\{ \hat{C}_j \}_{j=1}^{K}$ be the clusters formed over the stable set $\hat{\mathcal{Z}}$.
    \item Let $\{ \widetilde{C}_j \}_{j=1}^{K}$ be the clusters formed over the unstable set $\widetilde{\mathcal{Z}}$.
    \item $\hat{T}_S$ as the clusters with at least one stable training instance.
    \item $\widetilde{T}_S$ as the clusters with at least one unstable training instance.
\end{itemize}

Define the local robustness margin within each cluster as follows:
\begin{align*}
\hat{\epsilon}_j(\mathcal{A_S}) &:= \max_{z, s \in \hat{C}_j} \left| \ell(\mathcal{A_S}, z) - \ell(\mathcal{A_S}, s) \right|\\
\widetilde{\epsilon}_j(\mathcal{A_S}) &:= \max_{z, s \in \widetilde{C}_j} \left| \ell(\mathcal{A_S}, z) - \ell(\mathcal{A_S}, s) \right|
\end{align*}

Then, for any $\delta > 0$, with probability at least $1 - \delta$ over an iid draw of $n$ training samples, the generalization error of $\mathcal{A_S}$ is bounded as:
\begin{equation}
\begin{split}
    \mathcal{R}_Z(\mathcal{A_S}) \leq \hat{\mathcal{R}}_{\mathcal{S}}(\mathcal{A_S}) + \sum_{j \in \hat{T}_S} \frac{\hat{n}_j(\mathcal{A_S})}{n} \cdot \hat{\epsilon}_j(\mathcal{A_S}) \\
    + \sum_{j \in \widetilde{T}_S} \frac{\widetilde{n}_j(\mathcal{A_S})}{n} \cdot \widetilde{\epsilon}_j(\mathcal{A_S})  + u_{3}(K, \mathcal{S}, \delta)
\end{split}
\label{eqn:cluster_bound_both}
\end{equation}

where:
\begin{itemize}
    \item $\hat{n}_j(\mathcal{A_S})$ denotes the number of training samples in the stable cluster $\hat{C}_j$,
    \item $\widetilde{n}_j(\mathcal{A_S})$ denotes the number of training samples in the unstable cluster $\widetilde{C}_j$.
\end{itemize}
\end{theorem}


Theorem \ref{thm:cluster_bound} establishes that the robustness term in the previous bounds can be decomposed into separate contributions from stable and unstable samples within each cluster. For more general loss functions, this formulation underscores that stable samples, while exhibiting smaller errors \cite{7934087}, still make non-negligible contributions to the robustness term. To capture this distinction, we partition the training data into stable and unstable clusters: unstable clusters reflect regions of high sensitivity, whereas stable clusters account for consistently low-error regions.

This partition also provides insight into how different regions contribute to the overall expected risk. In particular, by decomposing the input space into stable and unstable regions, Lemma \ref{lem:lemma1} shows that, under $0$-$1$ loss, the contribution of the stable subset $\hat{\mathcal{Z}}$ to the expected risk is negligible, while the dominant contribution arises from the unstable subset $\widetilde{\mathcal{Z}}$.

Following \cite{7934087}, where it has been stated that robust models achieve a vanishing robustness term $\epsilon(S) = 0$ across the dataset, our stable subset $\hat{\mathcal{Z}}$ plays an analogous role at a local level. In these regions, the model $\mathcal{A}_S$ exhibits consistently correct predictions by construction (i.e., $\alpha_1 \approx 0$), which justifies focusing on the contribution of unstable clusters in Corollary~\ref{thm:ours}.

\begin{lemma}
\label{lem:lemma1}(Risk Bound via Unstable Samples)
Consider that the input space $\mathcal{Z}$ is decomposed into $\mathcal{Z}_1 = \hat{\mathcal{Z}}$ and $\mathcal{Z}_2 = \widetilde{\mathcal{Z}}$, where $\widetilde{\mathcal{Z}} = \mathcal{Z} \setminus \hat{\mathcal{Z}}$. Given 
$\alpha_i(A_S) := \mathbb{E}_{z \sim \mathcal{Z}_i}\big[\ell(A_S, z)\big],$ i.e., the expected loss restricted to the subset $\mathcal{Z}_i$. Then, the expected risk under $0$-$1$ loss can be decomposed as:
\[
\mathcal{R}_{\mathcal{Z}}(\mathcal{A}_S) 
= \mathbb{P}(z \in \mathcal{Z}_1)\,\alpha_1(\mathcal{A}_S) 
+ \mathbb{P}(z \in \mathcal{Z}_2)\,\alpha_2(\mathcal{A}_S).
\]

By construction, $\hat{\mathcal{Z}}$ represents the stable subset where the model $\mathcal{A}_S$ predicts correctly with high probability. Under the $0$-$1$ loss, correct predictions incur zero loss; therefore, the expected loss on $\mathcal{Z}_1$ is negligible, i.e., $\alpha_1(\mathcal{A}_S) \approx 0$. Moreover, since the $0$-$1$ loss is bounded in $[0,1]$, we have $\alpha_2(\mathcal{A}_S) \leq 1$. Consequently,
\[
\mathcal{R}_{\mathcal{Z}}(\mathcal{A}_S) 
\leq \mathbb{P}(z \in \mathcal{Z}_2)\,\alpha_2(\mathcal{A}_S)
\leq \alpha_2(\mathcal{A}_S).
\]
\end{lemma}
This implies that the expected risk can be bounded by the risk associated with the unstable samples in the dataset. Building on this lemma, we propose the following bounds.
    
\begin{corollary}[Per-Input Local Robustness]
\label{thm:ours}
Under the setting and notations of Theorem \ref{thm:cluster_bound}, assume that the expected loss over the stable subset $\hat{\mathcal{Z}}$ is negligible (e.g., under $0$-$1$ loss where correct predictions incur zero loss). Suppose the unstable subset of the sample space $\widetilde{\mathcal{Z}}$ can be clustered into $K$ disjoint subsets $\{ \widetilde{C}_j \}_{j=1}^{K}$.\\

Define:
\begin{itemize}
    \item the worst-case local robustness variation within each unstable cluster as:
\end{itemize}

\[
\widetilde{\epsilon}_j(\mathcal{A_S}) := \max_{\substack{s,z \in \widetilde{C}_j}} \left| \ell(\mathcal{A_S}, z) - \ell(\mathcal{A_S}, s) \right|.
\]

Then, for any $\delta > 0$, with probability at least $1 - \delta$ over the iid draw of the training dataset $\mathcal{S}$, the generalization error of $\mathcal{A_S}$ satisfies:
\begin{equation}
\begin{split}
    \mathcal{R}_Z(\mathcal{A_S}) \leq  \hat{\mathcal{R}}_{\mathcal{S}}(\mathcal{A_S})
    + \sum_{j \in \widetilde{T}_S} \frac{\widetilde{n}_j(\mathcal{A_S})}{n} \cdot \widetilde{\epsilon}_j(\mathcal{A_S}) + u_{3}(K, \mathcal{S}, \delta)
\end{split}
\label{eqn_2r_our}
\end{equation}
where:
\begin{itemize}
    \item $\widetilde{n}_j(\mathcal{A_S})$ is as defined in Theorem \ref{thm:cluster_bound}.
\end{itemize}
\end{corollary}

Corollary \ref{thm:ours} establishes that the expected loss of a model $\mathcal{A_S} \in \mathcal{H}$ can be bounded using $\sum_{j\in \widetilde{T}_S} \frac{\widetilde{n}_j(\mathcal{A_S})}{n} \cdot \widetilde{\epsilon}_j(\mathcal{A_S})$, which quantifies the local robustness of $\mathcal{A_S}$ based solely on unstable training instances across all the clusters. This result suggests that a model can achieve good generalization by maintaining local robustness specifically around unstable training points. By focusing on robustness efforts on these critical regions and ensuring a small empirical risk $\hat{\mathcal{R}}_{\mathcal{S}}(\mathcal{A_S})$, the model can attain a low expected loss over the entire sample space. Like the bound in (\ref{eqn_2r_2}), our bound in Corollary \ref{thm:ours} differs from the bounds in Theorems \ref{thm:xu-generalization} and \ref{thm:kawaguchi} in the following ways:
\begin{itemize}
    \item Firstly, our bound (\ref{eqn_2r_our}) (Corollary \ref{thm:ours}) relaxes the strict assumption of algorithmic robustness. Instead of requiring that a model exhibit uniformly small loss differences between all pairs of nearby training and test samples, our bound focuses only on the subset of training instances that exhibit unstable behavior.
    \item Secondly, our bound introduces both model specificity and data dependency in computing the robustness term. It explicitly depends on the particular model $\mathcal{A_S}$ and the training dataset $\mathcal{S}$. This is a clear advantage over prior bounds that rely only on global, data-dependent robustness terms. By capturing model-specific behavior, our bound provides a more accurate evaluation of generalization performance. In addition, it enables meaningful model selection and comparison across different architectures or training settings. This makes it especially useful for benchmarking robust learning methods in real-world applications.
    \item Thirdly, unlike prior bounds in (\ref{eqn_2r}) and (\ref{eqn_2r_1}) that rely on a global robustness level $\epsilon(\mathcal{S})$, our bound in \eqref{eqn_2r_our} uses localized terms $ \sum_{j \in \widetilde{T}_S} \frac{\widetilde{n}_j(\mathcal{S})}{n} \cdot \widetilde{\epsilon}_j(\mathcal{A_S})$, which focus on unstable training instances. Unstable samples correspond to regions where the model is highly sensitive, and this sensitivity is a primary source of variability across different stochastic realizations of the learning algorithm. By concentrating on these regions, the formulation alleviates the ``stochasticity inclusion'' challenge, where global robustness measures require worst-case aggregation over both stable and unstable regions as well as stochastic variations. Instead, it provides a more tractable and targeted characterization of generalization by focusing on the dominant sources of variability. While this approach is motivated by stochastic training effects, it does not explicitly model stochasticity, but mitigates its impact by avoiding overly conservative global estimates.
    \item Lastly, in real-world settings, generalization errors are often caused by a small subset of unstable or difficult training instances \cite{7934087}. The bound in \eqref{eqn_2r_our} accounts for this by weighting the worst-case margin according to the proportion of these unstable samples, resulting in a more accurate and tighter estimate of the expected loss. Worst-case generalization bounds that rely solely on unstable samples are tight under the $0$-$1$ loss, since assigning the maximum error to all unstable samples directly corresponds to classification errors.\\
    However, beyond the $0$-$1$ loss, our bound in Theorem \ref{thm:cluster_bound} provides a more suitable framework for assessing the generalization of deep learning models.
\end{itemize}
Next, remark \ref{lem:bound_comparison_revised} establishes the tightness of our bounds in Theorem \ref{thm:cluster_bound} and Corollary~\ref{thm:ours}. Finally, we relegate the proof of our bounds in Theorem \ref{thm:cluster_bound} and Corollary~\ref{thm:ours} to \ref{appendx:A}. \\

\begin{remark}[Comparison of Robustness Bounds]
\label{lem:bound_comparison_revised}

Under the same setting and notation as Theorems \ref{thm:cluster_bound} and \ref{thm:ours}, the following holds:

\begin{itemize}
    \item First, the local cluster margins $\hat{\epsilon}_j(\mathcal{A_S})$ for stable clusters are often zero or negligibly small due to inherent robustness within stable regions.
    \item Second, the per-input local robustness-based bound introduces finer granularity by scaling each intra-cluster supremum \(\widetilde{\epsilon}_j(\mathcal{A_S})\) by the individual cluster size \(\widetilde{n}_j (\mathcal{A_S})\). In contrast, the cluster stability-based bound captures the deviation stemming from the stables and unstable samples, thus making this bound general for different loss functions.
\end{itemize}

Hence, the per-input local robustness-based bound \eqref{eqn_2r_our} is often strictly tighter than the cluster stability-based bound \eqref{eqn:cluster_bound_both}. This is especially true for $0$-$1$ loss. It also holds when most stable clusters have volatile loss behavior and the worst-case deviations are limited to a few unstable regions. 
\end{remark}

\subsection{Our Bounds Comparison with Pseudo-Robustness Bounds}
Our proposed generalization bounds differ fundamentally from the pseudo-robustness framework introduced by \cite{xu2012robustness, kawaguchi2022robustness}, both in form and in practical applicability. Both bounds share a similar uncertainty term that captures the deviation from empirical to true risk due to finite sampling. However, they differ in how they use samples of interest and define robustness in the core decomposition. The pseudo-robustness bound relies on a data-dependent subset of ``stable'' samples determined by fixed input-space clustering. It introduces two terms: one that accounts for the robustness over these stable regions, and another that penalizes the remaining ``unstable'' samples via worst-case deviations. However, this formulation assumes that robustness can be captured purely through properties of the data, independent of how the model interacts with it \cite{xu2012robustness}.

In contrast, our bounds adopt a joint data- and model-dependent formulation. Specifically, we partition the sample set into $\hat{\mathcal{S}}$ and $\widetilde{\mathcal{S}}$, where the selection is informed not only by data structure but also by model behavior (e.g., local margins, confidence, or feature space concentration). The robustness terms $\hat{\epsilon}_j(\mathcal{A}_S)$ and $\widetilde{\epsilon}_j(\mathcal{A}_S)$ are similarly model-aware, capturing how the algorithm generalizes across these subsets. This allows the bound to better reflect the learning dynamics of complex models, particularly in high-capacity regimes such as deep neural networks.

Moreover, instead of relying on a single global robustness term across all clusters, our formulation distributes contributions proportionally through structured model-aware error terms. This prevents overestimation of instability, resulting in a tighter and more interpretable generalization bound. As a result, the proposed approach maintains computational tractability while more rigorously characterizing the behavior of modern learning models under real-world complexity and uncertainty.

\section{Implementation of the Proposed Framework}\label{section:4}

In this section, we present the proposed framework for empirically estimating the upper bounds on the true error of learned models using our derived bounds. Although the theoretical results are defined over the sample space $\mathcal{Z}$, their practical evaluation relies on finite datasets. Accordingly, we use the training dataset $\mathcal{S}$ and $Z_m$ as representative samples from $\mathcal{Z}$ to approximate the quantities involved, consistent with standard generalization assessment practices \cite{than2025gentle}. Our framework first partitions the training dataset into stable and unstable subsets, followed by a clustering procedure aligned with the bounds in \eqref{eqn:cluster_bound_both} and \eqref{eqn_2r_our}. Unlike conventional approaches, this explicit separation enables a more targeted estimation of generalization behavior by focusing on robustness characteristics within each subset.

To this end, we introduce the per-input resilient analyzer to perform the partitioning, and subsequently describe how stable and unstable clusters are constructed using the training dataset $\mathcal{S}$ together with the test samples $Z_m$.

\subsection{Per-input resilient analyzer}\label{per-input} The concept of using a per-input resilient analyzer to partition a dataset into stable and unstable subsets was first introduced by authors in \cite{nuhu2025validationstrategydeeplearning}. While they utilized the unstable samples to guide robustness enhancement, we examine how these subsets can be leveraged for generalization assessment. Algorithm \ref{algo:1} outlines the per-input resilient analyzer, which ranks data samples from least to most stable. The analyzer takes as input a trained model $\mathcal{A_S}$ and a given dataset $\mathcal{D}$. Empty dictionaries $\hat{\mathcal{D}}$ and $\widetilde{\mathcal{D}}$ are initialized to store the stable and unstable samples respectively. For each instance $x_i$ in the given dataset, the analyzer generates $\kappa$ perturbed neighbors through random sampling, with each perturbation constrained by a maximum perturbation magnitude $\epsilon$. Each perturbed neighbor $x_q$ (where $q \in \kappa$) is passed through $\mathcal{A_S}$ to obtain the predicted label $\hat{y}_q$. If the prediction is correct ($\hat{y}_q = y_{i}$), the neighbor is assigned a misprediction score $M_s = 0$; otherwise ($\hat{y}_q \neq y_{i}$), it is assigned $M_s = 1$. This scoring scheme is formalized in Equation~\eqref{misprediction}. The stability score $\gamma_i$ for the original sample $x_i$ is computed as the average of its neighbors’ misprediction scores as shown in Equation \eqref{eqn: gamma}. Once stability scores are computed for all instances in a class, the samples are sorted in descending order of $\gamma_i$, thereby ranking them from least to most stable. 
\begin{algorithm}
    \setlength{\textfloatsep}{5pt}  
    \small 
    \caption{Per-input Resilient Analyzer}
    \label{algo:1}
    \begin{algorithmic}[1]
        \Require \(\mathcal{A_S}\), \(\mathcal{D}\), \(\kappa\), \(\epsilon\), \(\tau\)
        \Ensure Lists \(\hat{\mathcal{D}}\) and \(\widetilde{\mathcal{D}}\)  containing stable and unstable training samples, respectively
        
        \Function{ComputeMispredictionScore}{$x$, $y$}
            \State \(neighbor\_sum \gets 0\)  
            \For{\(q = 1\) \textbf{to} \(\kappa\)}
                \State \(x_q \gets x + \epsilon \cdot \text{random}(\text{size}(x))\)  
                \State \(neighbor\_sum \gets neighbor\_sum + M_s(x_q)\)
            \EndFor
            \State \Return \(\frac{neighbor\_sum}{\kappa}\)
        \EndFunction
        
        \State \(\hat{\mathcal{D}} \gets [~]\), \(\widetilde{\mathcal{D}} \gets [~]\) 
        
        \For{each class \(y_i \in \mathcal{Y}\)}
            \State Find \(v_c\): samples in \(\mathcal{S}\) belonging to class \(y_i\)  
            \State Initialize list \(R \gets [~]\)  
            \For{\(i = 1\) \textbf{to} \(|v_c|\)}
                \State \(\gamma_i \gets \Call{ComputeMispredictionScore}{x_i, y_i}\)  
                \State Append \((x_i, \gamma_i)\) to \(R\)
            \EndFor
            \State \textbf{Sort} \(R\) in descending order by \(\gamma_i\)  
            
            \For{each \((x_i, \gamma_i) \in R\)}
                \If{\(\gamma_i \geq \tau\)}
                    \State Append \(x_i\) to \(\widetilde{\mathcal{D}}\)
                \Else
                    \State Append \(x_i\) to \(\hat{\mathcal{D}}\)
                \EndIf
            \EndFor
        \EndFor
        \State \Return \((\hat{\mathcal{D}}, \widetilde{\mathcal{D}})\)
    \end{algorithmic}
\end{algorithm}

Finally, based on a predefined threshold $\tau$, the sorted dataset is partitioned into stable and unstable subsets. Samples with stability scores greater than or equal to $\tau$ are assigned to the unstable subset $\widetilde{\mathcal{D}}$, while those with scores below $\tau$ are assigned to the stable subset $\hat{\mathcal{D}}$.\\ 
\begin{equation}
M_s(x_q) =
\begin{cases}
0, & \text{if } \hat{y}_q = y_i\\
1, & otherwise
\end{cases}
\label{misprediction}
\end{equation}
\begin{equation}
    \gamma_i = \frac{1}{\kappa} \sum_{q=1}^{\kappa} M_s(x_q)
    \label{eqn: gamma}
\end{equation}
\subsection{Stable and Unstable Clusters Formation}
To construct stable and unstable clusters, we first identify the corresponding samples in each dataset split (training or test) using per-input resilient analyzer described in subsection \eqref{per-input}. This partitions the training set $\mathcal{S}$ into stable samples $\hat{\mathcal{S}}$ and unstable samples $\widetilde{\mathcal{S}}$, and the test set $Z_m$ into stable samples $\hat{Z}$ and unstable samples $\widetilde{Z}$. The overall framework, along with the resulting cluster structure, is illustrated in Fig.~\ref{fig:part_clust}.
\begin{figure}
    \centering
    \includegraphics[width=1.0\columnwidth]{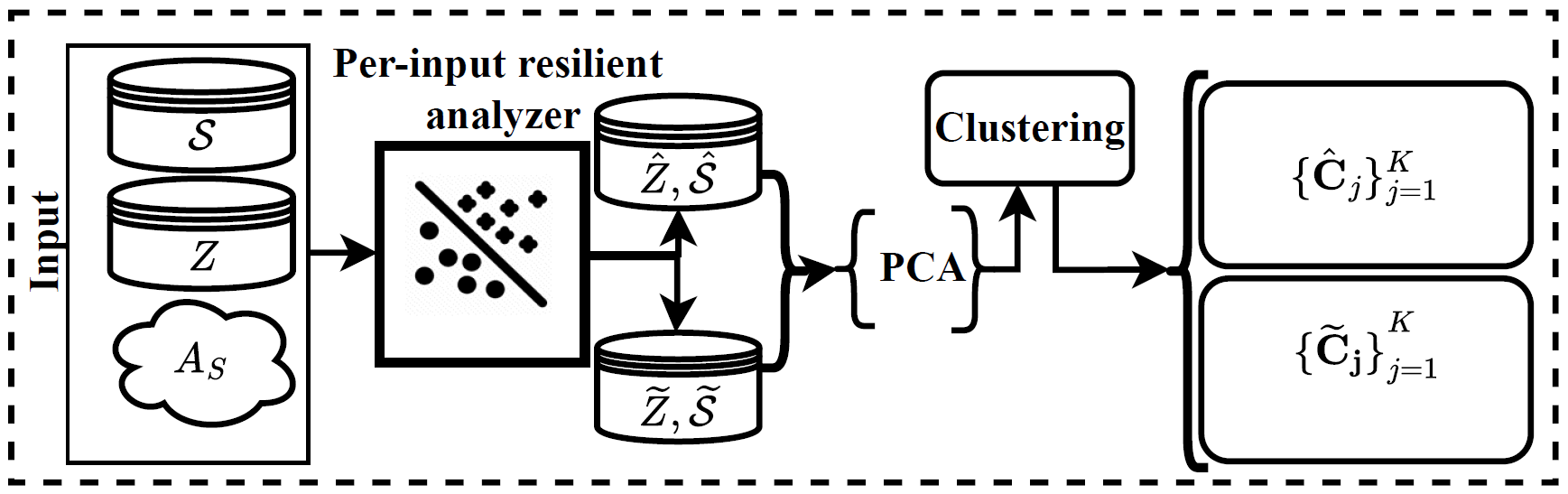}
    \caption{An overview of partitioning and clustering for the formation of stable and unstable clusters. $Z_m$ and $\mathcal{S}$ represent a test and train dataset, respectively.}
    \label{fig:part_clust}
\end{figure}
Cluster formation is performed independently for each category:

\begin{itemize}
    \item \textbf{Stable clusters:} Centroids are selected from the stable test subset $\hat{Z}$. High-level feature representations of both $\hat{Z}$ and $\hat{\mathcal{S}}$ are extracted using the trained model $\mathcal{A_S}$ and subsequently reduced via Principal Component Analysis (PCA). K-means clustering is then applied to assign each stable training sample in $\hat{\mathcal{S}}$ to its nearest stable centroid, forming stable clusters $\{\hat{C}_j\}_{j=1}^{K}$.
    
    \item \textbf{Unstable clusters:} Centroids are similarly selected from the unstable test subset $\widetilde{Z}$. PCA-reduced features of both $\widetilde{Z}$ and $\widetilde{\mathcal{S}}$ are computed, and K-means clustering is applied to assign each unstable training sample in $\widetilde{\mathcal{S}}$ to its nearest unstable centroid, forming unstable clusters $\{\widetilde{C}_j\}_{j=1}^{K}$.
\end{itemize}

By constructing stable and unstable clusters separately, this approach ensures that the resulting partitions truly capture the local behavior of the model in both low-variation (stable) and high-sensitivity (unstable) regions. Although the clusters are formed independently, the procedure is conceptually equivalent to decomposing a cluster into stable and unstable components. In both cases, the goal is to distinguish regions of consistent model behavior from those exhibiting high variability. This structured partitioning provides a rigorous basis for robustness-based analysis and localized generalization assessment. The complete procedure is summarized in Algorithm~\ref{algo:stable_unstable_clusters}.

\begin{algorithm}
    \setlength{\textfloatsep}{5pt}
    \small
    \caption{Stable and Unstable Clusters Formation}
    \begin{algorithmic}[1]
        \Require Training set \(\mathcal{S}\), test set \(Z\), model \(\mathcal{A_S}\), number of clusters \(K\)
        \Ensure Stable and unstable clusters: \(\{\hat{C}_j\}_{j=1}^{K}\), \(\{\widetilde{C}_j\}_{j=1}^{K}\)

        \State Use per-input resilient analyzer to partition \(\mathcal{S}\) into \(\hat{\mathcal{S}}\) (stable) and \(\widetilde{\mathcal{S}}\) (unstable)
        \State Similarly partition \(Z\) into \(\hat{Z}\) (stable) and \(\widetilde{Z}\) (unstable)

        \For{each subset in \(\{\hat{\mathcal{S}}, \widetilde{\mathcal{S}}\}\)}
            \State Extract high-level features using model \(\mathcal{A_S}\)
            \State Apply PCA to reduce feature dimensionality
            \If{subset is \(\hat{\mathcal{S}}\)}
                \State Extract PCA-reduced features from \(\hat{Z}\) to serve as centroids
                \State Apply K-means clustering to assign \(\hat{\mathcal{S}}\) into \(\{\hat{C}_j\}_{j=1}^{K}\)
            \Else
                \State Extract PCA-reduced features from \(\widetilde{Z}\) to serve as centroids
                \State Apply K-means clustering to assign \(\widetilde{\mathcal{S}}\) into \(\{\widetilde{C}_j\}_{j=1}^K\)
            \EndIf
        \EndFor
        \State \Return \(\{\hat{C}_j\}_{j=1}^K\), \(\{\widetilde{C}_j\}_{j=1}^K\)
    \end{algorithmic}
    \label{algo:stable_unstable_clusters}
\end{algorithm}
\subsection{Computational Complexity}
We briefly analyze the computational cost of Algorithms 1 and 2. For Algorithm 1 (Per-input Resilient Analyzer), let $n = |\mathcal{S}|$ denote the number of training samples and $\kappa$ the number of perturbations per sample. The dominant cost arises from evaluating the model over perturbed inputs, resulting in a total complexity of $O(n\kappa)$. In addition, the sorting operation incurs a cost of $O(n \log n)$. Therefore, the overall time complexity of Algorithm 1 is
\[
O(n\kappa + n \log n).
\]

Notably, the perturbation-based evaluations are independent across samples, allowing Algorithm 1 to be efficiently parallelized across multiple computational units (e.g., GPUs or distributed systems), which significantly improves scalability in practice. For Algorithm 2 (Stable and Unstable Clusters Formation), let $d$ denote the input feature dimension, $d'$ denote the extracted input feature dimension, $k$ the number of clusters, and $i$ the number of iterations in the K-means algorithm. The main computational steps include feature extraction, dimensionality reduction, and clustering. Feature extraction requires $O(n)$, while PCA incurs $O(n d^2)$ in general. The K-means clustering step has complexity $O(n K d' i)$. Thus, the overall time complexity of Algorithm 2 is
\[
O(n + n d^2 + n k d' i).
\]

Overall, both algorithms scale linearly with the dataset size $n$, with the dominant cost driven by perturbation-based evaluations in Algorithm 1 and clustering operations in Algorithm 2.

\section{Experimental Setup} \label{section:5}
In this section, we present an empirical evaluation of our proposed bounds and compare them against baseline methods using state-of-the-art deep learning models on standard classification datasets.

\subsection{Bounds evaluation on ImageNet pretrained models}
\subsubsection{Models Architecture Utilized}
To ensure a fair and consistent comparison of the generalization bounds, we adopt the same set of models pretrained on ImageNet \cite{russakovsky2015imagenet} utilized in the evaluation by \cite{than2025gentle}. Specifically, we employed 20 PyTorch-pretrained models\footnote{\url{https://pytorch.org/vision/stable/models.html}} drawn from four widely used deep neural network architectures: DenseNet \cite{Huang_2017_CVPR}, ResNet \cite{he2016deep}, Swin Transformer \cite{Liu_2021_ICCV}, and VGG \cite{simonyan2015a}. All models are trained on the ImageNet-1K dataset, which comprises 1,281,167 training samples.
\begin{figure*}
\centering
\subfloat[Accuracy of different models on stable and unstable samples from the training dataset.]{\includegraphics[width=0.99\linewidth]{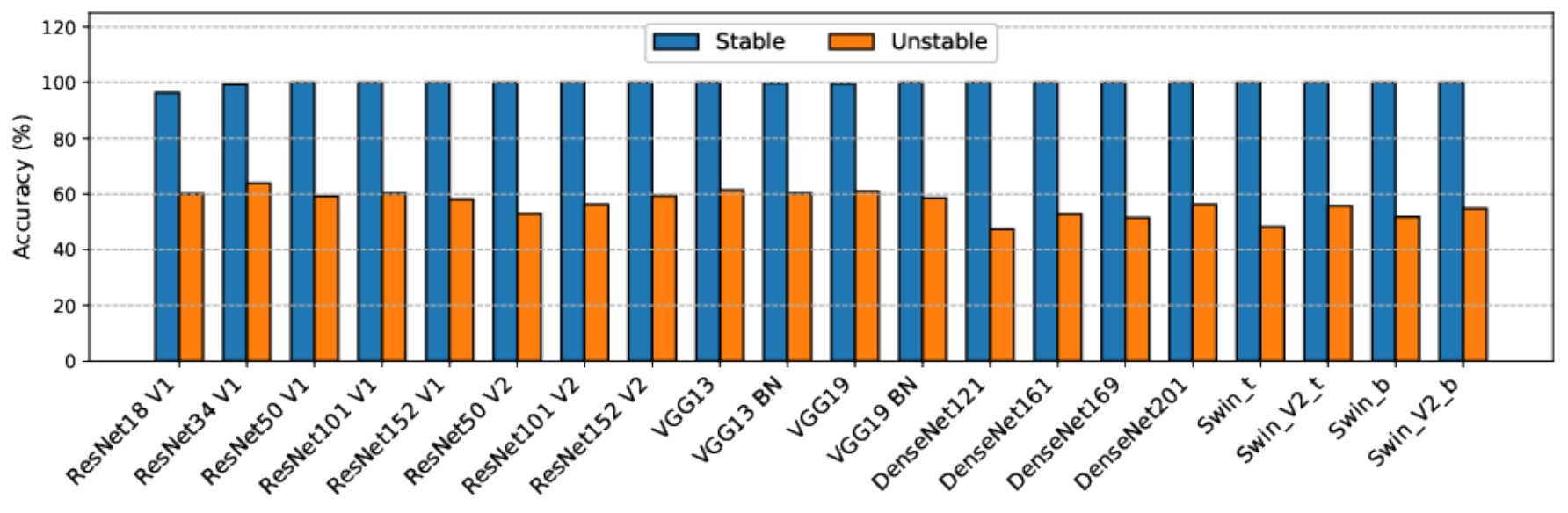}%
\label{fig:sable_vs_unstable_accuracies}}
\vfil
\subfloat[Accuracy of different models on stable and unstable samples from the validation dataset.]{\includegraphics[width=0.99\linewidth]{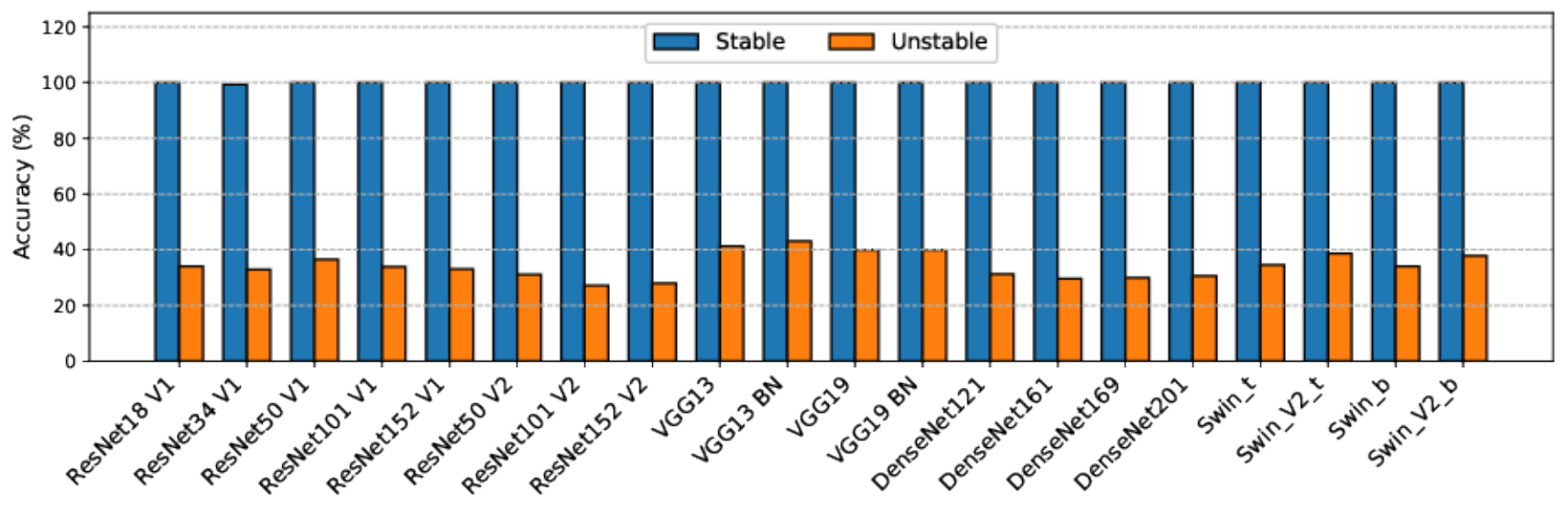}%
\label{fig:sable_vs_unstable_accuracies_val}}
\caption{Performance of different models on stable and unstable samples from the training dataset (a) and validation dataset (b). While stable samples yield near-perfect accuracies, unstable samples reveal substantial performance degradation, highlighting their utility as a proxy for robustness-based generalization assessment under the $0$–$1$ loss.}
\label{fig:stable_vs_unstable_combined}
\end{figure*}
\subsubsection{Implementation details} To group the training dataset into the stable and unstable samples we set the following parameters as follow: We perform a sensitivity analysis for $\kappa$ by varying its value between $10$ and $200$. Following \cite{10316046, zhong2021understanding}, we set 
$\kappa = 20$, since the average numbers of stable and unstable samples stabilize at this point and remain consistent for larger values. This ensures reliable results while avoiding extra computational overhead. For the perturbation degree $\epsilon$, we adopt the value $0.1$, as suggested in \cite{kawaguchi2022generalization}, for robustness analysis. We set the misprediction threshold $\tau$ to $0.2$ ($20\%$), meaning that a sample is considered stable if at least $80\%$ of its neighbors are predicted correctly. This choice strikes a balance between sensitivity and specificity \cite{10316046}, ensuring that unstable samples capture meaningful misprediction patterns without being overly restrictive. While $\tau = 0.2$ is effective in our setting, the threshold can be further tuned to optimize performance for different datasets or models.
\begin{table*}
    \centering
    \caption{Different generalization bounds comparison for state-of-the-art models trained on ImageNet. Bounds given by \eqref{eqn_2r_1}, \eqref{eqn_2r_2}, \eqref{eqn_2r_3}, \eqref{eqn:cluster_bound_both}, and \eqref{eqn_2r_our} correspond to distinct theoretical guarantees that upper bound the true error, Prob($\mathcal{A_S}(x) \neq y$). Lower values ($\downarrow$) indicate better performance. Best and second-best results are shown in bold and underlined, respectively. Std values from our bounds are consistently $<10^{-3}$ and reported as such for consistency with prior work.}
    \vspace{\floatsep}
    \setlength{\tabcolsep}{6pt} 
    \renewcommand{\arraystretch}{1.2} 
    \begin{tabular}{lcccccc}
        \toprule
        \textbf{Model} & \textbf{Error@1} & \textbf{Bound \eqref{eqn_2r_1}} & 
        \textbf{Bound \eqref{eqn_2r_2}} & \textbf{Bound \eqref{eqn_2r_3}} & 
        \textbf{Bound \eqref{eqn:cluster_bound_both}} & 
        \textbf{Bound \eqref{eqn_2r_our}} \\
        \midrule
        ResNet18 V1 & 0.302 & 1.527 $\pm$ 0.005 & 1.527 $\pm$ 0.005 & 0.917 $\pm$ 0.005 & \underline{\textit{$0.906 \,\pm\, (<10^{-3})$}} & $\mathbf{0.869 \,\pm\, (<10^{-3})}$ \\
        ResNet34 V1 & 0.267 & 1.462 $\pm$ 0.005 & 1.462 $\pm$ 0.005 & 0.805 $\pm$ 0.004 & \underline{\textit{$0.767\,\pm\, (<10^{-3})$}} & $\mathbf{0.760\,\pm\, (<10^{-3})}$ \\
        ResNet50 V1 & 0.239 & 1.431 $\pm$ 0.004 & 1.430 $\pm$ 0.004 & 0.743 $\pm$ 0.004 & \underline{\textit{$0.712 \,\pm\,(<10^{-3})$}} & $\mathbf{0.712 \,\pm\, (<10^{-3})}$ \\
        ResNet101 V1 & 0.226 & 1.401 $\pm$ 0.005 & 1.400 $\pm$ 0.005 & 0.688 $\pm$ 0.005 & \underline{\textit{$0.628 \,\pm\, (<10^{-3})$}}  & $\mathbf{0.628 \,\pm\, (<10^{-3})}$ \\
        ResNet152 V1 & 0.217 & 1.395 $\pm$ 0.004 & 1.394 $\pm$ 0.004 & 0.673 $\pm$ 0.004 & \underline{\textit{$0.604 \,\pm\, (<10^{-3})$}} & $\mathbf{0.604 \,\pm\, (<10^{-3})}$ \\
        ResNet50 V2 & 0.191 & 1.379 $\pm$ 0.004 & 1.377 $\pm$ 0.005 & 0.633 $\pm$ 0.004 & \underline{\textit{$0.544 \,\pm\, (<10^{-3})$}} & $\mathbf{0.544 \,\pm\, (<10^{-3})}$ \\
        ResNet101 V2 & 0.181 & 1.346 $\pm$ 0.004 & 1.344 $\pm$ 0.004 & 0.571 $\pm$ 0.004 & \underline{\textit{$0.502 \,\pm\, (<10^{-3})$}} & $\mathbf{0.502 \,\pm\, (<10^{-3})}$ \\
        ResNet152 V2 & 0.177 & 1.337 $\pm$ 0.004 & 1.333 $\pm$ 0.004 & 0.552 $\pm$ 0.004 & \underline{\textit{$0.485 \,\pm\, (<10^{-3})$}} & $\mathbf{0.485 \,\pm\, (<10^{-3})}$ \\
        VGG13 & 0.301 & 1.500 $\pm$ 0.005 & 1.499 $\pm$ 0.005 & $\mathbf{0.879 \pm 0.005}$ & $0.915 \,\pm\, (<10^{-3})$ & \underline{\textit{$0.914 \,\pm\, (<10^{-3})$}} \\
        VGG13 BN & 0.284 & 1.504 $\pm$ 0.004 & 1.503 $\pm$ 0.005 & $\mathbf{0.876 \pm 0.004}$ & $0.916 \,\pm\, (<10^{-3})$ & \underline{\textit{$0.914 \,\pm\, (<10^{-3})$}} \\
        VGG19 & 0.276 & 1.470 $\pm$ 0.005 & 1.469 $\pm$ 0.005 & $\mathbf{0.821 \pm 0.005}$ & $0.828 \,\pm\, (<10^{-3})$ & \underline{\textit{$0.825 \,\pm\, (<10^{-3})$}} \\
        VGG19 BN & 0.258 & 1.464 $\pm$ 0.004 & 1.463 $\pm$ 0.005 & 0.803 $\pm$ 0.004 & \underline{\textit{$0.789 \,\pm\, (<10^{-3})$}} & $\mathbf{0.789 \,\pm\, (<10^{-3})}$ \\
        DenseNet121 & 0.256 & 1.457 $\pm$ 0.005 & 1.457 $\pm$ 0.005 & 0.785 $\pm$ 0.005 & \underline{\textit{$0.739 \,\pm\, (<10^{-3})$}} & $\mathbf{0.739 \,\pm\, (<10^{-3})}$ \\
        DenseNet161 & 0.229 & 1.400 $\pm$ 0.004 & 1.398 $\pm$ 0.004 & 0.681 $\pm$ 0.004 & \underline{\textit{$0.612 \,\pm\, (<10^{-3})$}} & $\mathbf{0.612 \,\pm\, (<10^{-3})}$ \\
        DenseNet169 & 0.244 & 1.422 $\pm$ 0.004 & 1.421 $\pm$ 0.004 & 0.725 $\pm$ 0.004 & \underline{\textit{$0.665 \,\pm\, (<10^{-3})$}} & $\mathbf{0.665 \,\pm\, (<10^{-3})}$ \\
        DenseNet201 & 0.231 & 1.393 $\pm$ 0.004 & 1.392 $\pm$ 0.004 & 0.673 $\pm$ 0.005 & \underline{\textit{$0.562 \,\pm\, (<10^{-3})$}} & $\mathbf{0.562 \,\pm\, (<10^{-3})}$ \\
        Swin\_b & 0.164 & 1.347 $\pm$ 0.004 & 1.345 $\pm$ 0.004 & 0.563 $\pm$ 0.004 & \underline{\textit{$0.466 \,\pm\, (<10^{-3})$}} & $\mathbf{0.466 \,\pm\, (<10^{-3})}$ \\
        Swin\_t & 0.185 & 1.389 $\pm$ 0.004  & 1.387 $\pm$ 0.004 & 0.647 $\pm$ 0.004 & \underline{\textit{$0.536 \,\pm\, (<10^{-3})$}} & $\mathbf{0.536 \,\pm\, (<10^{-3})}$ \\
        Swin\_V2\_b & 0.159 & 1.345 $\pm$ 0.004 & 1.342 $\pm$ 0.004 & 0.551 $\pm$ 0.004 & \underline{\textit{$0.461 \,\pm\, (<10^{-3}$})} & $\mathbf{0.461 \,\pm\, (<10^{-3})}$ \\
        Swin\_V2\_t  & 0.179 & 1.373 $\pm$ 0.004 & 1.372 $\pm$ 0.004 &  0.613 $\pm$ 0.004 & \underline{\textit{$0.534 \,\pm\, (<10^{-3})$}} & $\mathbf{0.534 \,\pm\, (<10^{-3})}$ \\
        \bottomrule
    \end{tabular}
    \label{tab:bound_comparison}
\end{table*}
To evaluate the bounds in \eqref{eqn:cluster_bound_both} and \eqref{eqn_2r_our}, we follow the methodology proposed in~\cite{kawaguchi2022robustness} by partitioning the input space into 10{,}000 regions. To compute \( \widetilde{\epsilon}_j(\mathcal{A}_S) \) and \( \epsilon^*(\mathcal{A}_S) \) in \eqref{globalmax_epsilon}, we first randomly select 10{,}000 unstable validation samples to serve as centroids for forming unstable clusters. Using these centroids and following the procedure outlined in algorithm \ref{algo:stable_unstable_clusters}, we generated the unstable clusters $\{\widetilde{C}_j\}_{j=1}^{K}$. For each cluster \( j \in K \), we compute \( \widetilde{\epsilon}_j(\mathcal{A}_S)\), and subsequently compute \( \epsilon^*(\mathcal{A}_S) \). Similarly, to compute \( \hat{\epsilon}_j(\mathcal{A}_S) \), we randomly select 10{,}000 stable validation samples as centroids and following the procedure outlined in algorithm \ref{algo:stable_unstable_clusters}, we formed the stable clusters $\{\hat{C}_j\}_{j=1}^{K}$. For each cluster \( j \in K \), we compute \( \hat{\epsilon}_j(\mathcal{A}_S) \). For bounds in Equations~\eqref{eqn_2r_1}, \eqref{eqn_2r_2} and  \eqref{eqn_2r_3}, we reported the values provided in \cite{than2025gentle}. Following \cite{than2025gentle}, we set $\delta=0.01 (99\%~\text{confidence intervals})$ to compute the bounds. Also, we used $0-1$ loss function and as usual any bound beyond $1$ will be vacuous. The results for each model in Tables \ref{tab:bound_comparison} and \ref{tab:expected risk estimate} are averaged over five independent runs. For each model, we report both the mean and the standard deviation (std). In Table \ref{tab:bound_comparison}, for all models, the standard deviation (std) of our bounds is smaller than three decimal places. To ensure a fair comparison with existing results, we report these values as $(<10^{-3})$. Finally, we excluded the generalization bound presented in Theorem 6 of \cite{than2025gentle} in our empirical comparison due to fundamental differences in formulation. Specifically, their bound eliminates the empirical risk term and instead relies on an average loss-based robustness measure computed over local regions. In contrast, our bound explicitly retains the empirical risk term and introduces a localized robustness characterization that distinguishes between stable and unstable regions. These structural differences lead to distinct interpretations of robustness and generalization behavior, making a direct quantitative comparison non-equivalent. Therefore, our evaluation focuses on methods with comparable formulations to ensure a consistent and meaningful comparison.
\subsection{Results and discussion}
Figure~\ref{fig:sable_vs_unstable_accuracies} shows model performance on stable and unstable samples from the training dataset using $\kappa=20$. Additional sensitivity analysis over randomly selected values of $\kappa$ is provided in Appendix~\ref{Appendix_B}. Specifically, we report the percentage of unstable samples for each model across five randomly selected values of $\kappa$. Most models achieve near-perfect accuracy ($\sim 100\%$) on stable samples, but performance drops substantially on unstable samples. Although drawn from the training set, these unstable samples are more challenging to classify as they lie near decision boundaries and exhibit high sensitivity to small perturbations~\cite{NEURIPS2024_29753d93}. This property makes them an effective proxy for evaluating generalization. Stable samples capture confident predictions, whereas unstable samples expose the model’s behavior on challenging cases. The resulting performance gap underscores the importance of unstable samples for robustness-based generalization assessment. A similar trend is observed across all model architectures on the validation dataset, as shown in Figure~\ref{fig:sable_vs_unstable_accuracies_val}.

Table \ref{tab:bound_comparison} compares prior bounds (Equations~(\ref{eqn_2r_1}, \ref{eqn_2r_2}, \ref{eqn_2r_3})) with our proposed bounds (Equations~(\ref{eqn:cluster_bound_both}, \ref{eqn_2r_our})). The bound in \eqref{eqn_2r_our} is a special case of the more general bound in \eqref{eqn:cluster_bound_both}. We present both to show that, while their differences are minimal in practice, they become important when considering specific loss functions. In particular, the bound in \eqref{eqn_2r_our} is better suited for robustness-based generalization assessment under the $0$-$1$ loss. We observe that the prior bounds in \eqref{eqn_2r_1} and \eqref{eqn_2r_2} yield vacuous results across all models. In contrast, the bound in \eqref{eqn_2r_3} produces non-vacuous values; however, it relies on the average robustness value across clusters rather than the worst-case robustness within each region. Averaging can obscure weaknesses, making a model appear more reliable than it actually is, in terms of the generalization estimates. By contrast, worst-case robustness at clusters exposes these hidden vulnerabilities, leading to a more reliable generalization estimate. Our bound in \eqref{eqn_2r_our}, which leverages worst-case robustness at clusters, is non-vacuous in all cases and produces significantly better results, except for VGG13, VGG13 BN, and VGG19. In these VGG models, the average-loss approach in \eqref{eqn_2r_3} smooths over rare, extreme failures, yielding a tighter but overly optimistic bound. Additionally, our bound tends to perform particularly well for models with lower error rates, such as ResNet V2 and SwinTransformers. These models have fewer unstable regions. As a result, this reduces the impact of the worst-case measure, keeping the bound tight while still capturing robustness. For models with higher error, unstable regions dominate, and the worst-case approach can naturally produce looser but more reliable bounds. Finally, it can be observed that our bound in \eqref{eqn_2r_our} supports the point that, in the case of $0$-$1$ loss, unstable samples alone are sufficient for generalization assessment. Indeed, except for a few cases, bounds \eqref{eqn:cluster_bound_both} and \eqref{eqn_2r_our} converge to the same values.

Building on standard generalization assessment methods, we connect our theoretical bounds to practical estimation procedures by evaluating how well these bounds approximate the true risk of a trained model when the uncertainty terms are omitted. Specifically, the main generalization bounds, including Bounds \eqref{eqn_2r_1} - \eqref{eqn_2r_our}, provide formal upper bounds on the expected risk over the sample space. Equations \eqref{rob}–\eqref{localavg} correspond to existing methods that approximate the known test error, while Equations \eqref{localmax}–\eqref{globalmax} implement our proposed estimation approach based on localized robustness. In particular, our method highlights how unstable samples can effectively approximate the upper bound of the model’s expected risk on the known test set $Z_m$, thereby providing a practical instantiation of the theoretical bounds.
\begin{equation}
    \text{Rob} = \hat{\mathcal{R}}_{\mathcal{S}}(\mathcal{A_S}) +  \epsilon(\mathcal{S})
    \label{rob}
\end{equation}
\begin{equation}
    LocalRob = \hat{\mathcal{R}}_{S}(\mathcal{A_S}) +  \frac{n_j}{n}{\epsilon}_j(\mathbf{h}) 
    \label{localrob}
\end{equation}
\begin{equation}
    \text{LocalSen} = \hat{\mathcal{R}}_{\mathcal{S}}(\mathcal{A_S}) +  \sum_{j\in T_S}\frac{n_j}{n}\bar{\epsilon}_j(\mathbf{h})
    \label{localavg}
\end{equation}

\begin{equation}
    \text {LocalMax} = \hat{\mathcal{R}}_{\mathcal{S}}(\mathcal{A_S}) +  \sum_{j\in \widetilde{T}_S} \frac{\widetilde{n}_j(\mathcal{A_S})}{n}\widetilde{\epsilon}_j(\mathcal{A_S})
    \label{localmax}
\end{equation}

\begin{equation}
    \text{GlobalMax} = \hat{\mathcal{R}}_{\mathcal{S}}(\mathcal{A_S}) +  \sum_{j\in \widetilde{T}_S} \frac{\widetilde{n}(\mathcal{A_S})}{n}\epsilon^*(\mathcal{A_S})
    \label{globalmax}
\end{equation}
where:
\begin{align}
    \epsilon^*(\mathcal{A_S}) := \max_{j \in [K]} \max_{\substack{s, z \in \widetilde{C}_j }} \left| \ell(\mathcal{A_S}, z) - \ell(\mathcal{A_S}, s) \right|
   \label{globalmax_epsilon}
\end{align}
\begin{table*}
    \centering
    \caption{Upper generalization bounds on the true error (expected risk) (Prob(\(\mathcal{A_S}(x) \neq y\))) of different deep neural networks (DNN) pretrained on the ImageNet dataset. The upper bounds estimates on the Error@1 is provided without considering the uncertain terms. Bold entries denote the best result for each model, and underlined entries indicate the second best.}
    \vspace{\floatsep}
    \setlength{\tabcolsep}{6pt} 
    \renewcommand{\arraystretch}{1.2} 
    \begin{tabular}{lccccc}
        \toprule
        \textbf{Model} & \textbf{Err@1} & \textbf{Rob $\downarrow$} & 
        \textbf{LocalSen $\downarrow$} & \textbf{GlobalMax $\downarrow$} & 
        \textbf{LocalMax $\downarrow$} \\
        \midrule
        ResNet18 V1 & 0.302 & $1.212 \pm 4.0\times 10^{-5}$ & $0.602 \pm 1.9\times 10^{-4}$ &  \underline{\textit{$0.593\pm 1.4\times 10^{-5}$}}  & $\mathbf{0.591 \pm 1.1\times 10^{-5}}$   \\
        ResNet34 V1 & 0.267 & $1.157 \pm 4.0\times 10^{-5}$ & $0.499 \pm 1.8\times 10^{-4}$ &  \underline{\textit{$0.455\pm 2.1\times 10^{-5}$}}  & $\mathbf{0.454 \pm 1.1\times 10^{-5}}$  \\
        ResNet50 V1 & 0.239 & $1.131 \pm 6.0\times 10^{-5}$ & $0.443 \pm 3.0\times 10^{-4}$ &  \underline{\textit{$0.413\pm 2.3\times 10^{-5}$}}  & $\mathbf{0.412 \pm 1.6\times 10^{-5}}$ \\
        ResNet101 V1 & 0.226 & $1.105 \pm 5.0\times 10^{-5}$ & $0.392 \pm 1.7\times 10^{-4}$ &  \underline{\textit{$0.334 \pm 1.8\times 10^{-5}$}}  & $\mathbf{0.332 \pm 4.9\times 10^{-5}}$ \\
        ResNet152 V1 & 0.217 & $1.101 \pm 4.0\times 10^{-5}$ & $0.379 \pm 2.4\times 10^{-4}$ &  \underline{\textit{$0.314 \pm 3.0\times 10^{-5}$}}  & $\mathbf{0.310 \pm 4.3\times 10^{-5}}$ \\
        ResNet50 V2 & 0.191 & $1.089 \pm 4.0\times 10^{-5}$ & $0.344 \pm 3.5\times 10^{-4}$ &  \underline{\textit{$0.258\pm 1.0\times 10^{-5}$}} & $\mathbf{0.255 \pm 3.6\times 10^{-5}}$  \\
        ResNet101 V2 & 0.181 & $1.060 \pm 2.0\times 10^{-5}$ & $0.285 \pm 3.0\times 10^{-4}$ & \underline{\textit{$0.216 \pm 2.6\times 10^{-5}$}} & $\mathbf{0.216 \pm 1.4\times 10^{-5}}$ \\
        ResNet152 V2 & 0.177 & $1.052 \pm 4.0\times 10^{-5}$ & $0.267 \pm 2.5\times 10^{-4}$ & \underline{\textit{$0.200 \pm 1.6\times 10^{-5}$}}  & $\mathbf{0.200 \pm 1.1\times 10^{-5}}$\\
        VGG13 & 0.301 & $1.184 \pm 5.0\times 10^{-5}$ & $\mathbf{0.563 \pm 2.6\times 10^{-4}}$ &  $0.601 \pm 6.5\times 10^{-5}$ & \underline{\textit{$0.599 \pm 1.1\times 10^{-5}$}} \\
        VGG13 BN & 0.284 & $1.192 \pm 4.0\times 10^{-5}$ & $\mathbf{0.564 \pm 2.8\times 10^{-4}}$ & $0.605 \pm 5.7\times 10^{-5}$ & \underline{\textit{$0.602 \pm 1.7\times 10^{-5}$}}  \\
        VGG19 & 0.276 & $1.161 \pm 6.0\times 10^{-5}$ & $\mathbf{0.512 \pm 3.3\times 10^{-4}}$ & $0.520 \pm 4.5\times 10^{-5}$ & \underline{\textit{$0.516 \pm 2.4\times 10^{-5}$}}\\
        VGG19 BN & 0.258 & $1.159 \pm 4.0\times 10^{-5}$ & $0.499 \pm 2.9\times 10^{-4}$ & \underline{\textit{$0.488 \pm 5.9\times 10^{-5}$}} & $\mathbf{0.485 \pm 1.8\times 10^{-5}}$ \\
        DenseNet121 & 0.256 & $1.156 \pm 4.0\times 10^{-5}$ & $0.484 \pm 1.9\times 10^{-4}$ & \underline{\textit{$0.440 \pm 4.1\times 10^{-5}$}} & $\mathbf{0.438 \pm 2.1\times 10^{-5}}$  \\
        DenseNet161 & 0.229 & $1.105 \pm 4.0\times 10^{-5}$ & $0.386 \pm 1.9\times 10^{-4}$ & \underline{\textit{$0.318 \pm 1.9\times 10^{-5}$}} & $\mathbf{0.317 \pm 1.0\times 10^{-5}}$ \\
        DenseNet169 & 0.244 & $1.124 \pm 4.0\times 10^{-5}$ & $0.427 \pm 1.6\times 10^{-4}$ & \underline{\textit{$0.368 \pm 4.3\times 10^{-5}$}} & $\mathbf{0.367 \pm 4.1\times 10^{-5}}$ \\
        DenseNet201 & 0.231 & $1.098 \pm 4.0\times 10^{-5}$ & $0.378 \pm 2.1\times 10^{-4}$ & \underline{\textit{$0.316 \pm 2.6\times 10^{-5}$}} & $\mathbf{0.312 \pm 1.8\times 10^{-5}}$ \\
        SwinTransformer B & 0.164 & $1.065 \pm 4.0\times 10^{-5}$ & $0.280 \pm 1.0\times 10^{-4}$ & \underline{\textit{$0.186 \pm 1.8\times 10^{-5}$}} & $\mathbf{0.183 \pm 2.4\times 10^{-5}}$\\
        SwinTransformer T & 0.185 & $1.100 \pm 4.0\times 10^{-5}$ & $0.358 \pm 3.3\times 10^{-4}$ & \underline{\textit{$0.298 \pm 3.7\times 10^{-5}$}} & $\mathbf{0.247 \pm 3.2\times 10^{-5}}$\\
        SwinTransformer B V2 & 0.159 & $1.064 \pm 2.0\times 10^{-5}$ & $0.270 \pm 2.4\times 10^{-4}$ & \underline{\textit{$0.186 \pm 5.3\times 10^{-5}$}} & $\mathbf{0.180 \pm 4.2\times 10^{-5}}$\\
        SwinTransformer T V2 & 0.179 & $1.087 \pm 4.0\times 10^{-5}$ & $0.327 \pm 3.8\times 10^{-4}$ & \underline{\textit{$0.250 \pm 6.6\times 10^{-5}$}} & $\mathbf{0.248 \pm 2.5\times 10^{-5}}$\\
        \bottomrule
    \end{tabular}
    \vspace{\floatsep}
    \label{tab:expected risk estimate}
\end{table*}
Rob and LocalSen comes from prior work \cite{than2025gentle} whereas LocalMax comes from our bounds.
GlobalMax is a conservative form of the LocalMax approach and it applies the single worst-case instability across all unstable samples. In contrast, LocalMax evaluates worst-case robustness within each cluster, weighted by the number of unstable samples, which yields a tighter and more precise estimate. While LocalMax is better suited for detailed expected risk estimation, GlobalMax remains useful as a quick coarse measure, providing an upper bound that complements LocalMax through fast estimation. Table \ref{tab:expected risk estimate} summarizes the results of different approaches for estimating the expected risk (i.e., the known test error) of the models. The estimate corresponding to bound \eqref{eqn_2r_2} (Eqn. \eqref{localrob}) is equivalent to that of the Rob bound \eqref{rob}, which is derived from the base robustness-based generalization bound. Therefore, to avoid redundancy, we report only the results for the Rob bound. Both GlobalMax and LocalMax provide accurate estimates for most models when compared to LocalSen, with the exception of VGG13, VGG13 BN, and VGG19. This difference arises from the choice of robustness terms: LocalSen relies on the average local robustness within a cluster, whereas our approaches employ the worst-case robustness term. Because our method adopts a worst-case robustness perspective. Consequently, stable samples that are misclassified during the partitioning stage are classified as unstable. This increases the number of samples contributing to the robustness term. In such cases, the resulting bounds may be looser than those derived from prior approaches.

\section{Conclusion}\label{section:6}
In this paper, we examined prior work on robustness-based generalization bounds and demonstrated empirically that their vacuousness can also stem from the robustness term. To address this limitation, we developed tighter robustness-based bounds by partitioning the input space into stable and unstable samples, and further refining these into smaller clusters. Based on the model’s local behavior, we introduced a localized robustness term, decomposed into stable and unstable components. Our evaluations reveal that unstable samples are the dominant contributors to the robustness term. For classification tasks under the $0$–$1$ loss, prior bounds \cite{kawaguchi2022robustness} remain vacuous, while our proposed bound closely approximates the true error of classifiers. Unlike earlier approaches that rely solely on a data-dependent global robustness term, our bounds incorporate model properties to guide robustness, making them both model- and data-dependent. Moreover, our method relaxes the strict definition of robustness that traditionally constrains generalization assessment. As a result, our bounds yield sharper and more reliable generalization estimates, while also serving as effective tools for model comparison and selection.

This work is not without limitations. Effectiveness of our bounds depends on the quality of the partitioning strategy. If the approach incorrectly classifies certain stable samples as unstable, the resulting bounds may become loose. Nonetheless, this limitation can be mitigated, as recent advances in partitioning methods \cite{NEURIPS2024_29753d93} have shown strong effectiveness in handling training datasets.
As a direction for future research, our treatment of the robustness term in robustness-based generalization bounds can be extended to refine the uncertainty component. Prior work \cite{kawaguchi2022robustness} has demonstrated that the training dataset alone is sufficient to bound this uncertainty term. Building on this insight, a promising extension would be to decompose the uncertainty into stable and unstable components, which may yield tighter and more informative bounds. Finally, we plan to investigate how different loss functions influence the generalization behavior of various model architectures.\\  

\section*{Acknowledgment}
This work was supported in part by NASA ULI under Grants No. 80NSSC20M0161 and 80NSSC25M7098, by the National Science Foundation under Grant No. 2301553, by the U.S. Department of Transportation University Transportation Center under Grant No. 69A3552348327, and by NC-DOT AAM and UAS under Grant No. RP 2025-43.

\appendix \label{appendix}
\section{\textbf{Proof of Theorem~\ref{thm:cluster_bound}, Lemma~\ref{lem:lemma1} and Corollary~\ref{thm:ours}}}\label{appendx:A}
\begin{proof}[Proof of Theorem \ref{thm:cluster_bound}]
for any $\delta > 0$, with probability at least $1 - \delta$ over an iid draw of $n$ training samples, the generalization error of $\mathcal{A_S}$ is bounded as:
\begin{align}
&\mathcal{R}_Z(\mathcal{A_S}) \notag \leq \hat{\mathcal{R}}_{\mathcal{S}}(\mathcal{A_S})
+ \sum_{j \in \hat{T}_S}
\frac{\hat{n}_j(\mathcal{A_S})}{n}
\cdot
\hat{\epsilon}_j(\mathcal{A_S}) \notag \\
&\quad
+ \sum_{j \in \widetilde{T}_S}
\frac{\widetilde{n}_j(\mathcal{A_S})}{n}
\cdot
\widetilde{\epsilon}_j(\mathcal{A_S})
+ u_{3}(K, \mathcal{S}, \delta)
\end{align}

To present the proof of our main theorem, we start with the following observations. We define 
\[
\hat{n}_j := I^{\hat{S}}_{j} := \{\, i \in [\hat{n}(\mathcal{A_S})] \mid z_i \in \hat{C}_j \,\}
\]
and 
\[
\widetilde{n}_j := I^{\widetilde{S}}_{j} := \{\, i \in [\widetilde{n}(\mathcal{A_S})] \mid z_i \in \widetilde{C}_j \,\}
\]
and 

\begin{align*}
\alpha_j(h) &:= \mathbb{E}_{z}\big[ \ell(h, z) \,\big|\, z \in C_j \big] \notag \\
            &\;= \mathbb{E}_{\hat{z}}\big[ \ell(h, \hat{z}) \,\big|\, \hat{z} \in \hat{C}_j \big] + \mathbb{E}_{\widetilde{z}}\big[ \ell(h, \widetilde{z}) \,\big|\, \widetilde{z} \in \widetilde{C}_j \big] \notag \\
            &\;= \hat{\alpha}_j(h) + \widetilde{\alpha}_j(h)
\label{eq:alpha_decomposition}
\end{align*}
where 
\[ 
\hat{C}_j \cup  \widetilde{C}_j = C_j
\]

We begin the proof with the following lemma that relates the generation gap to the concentration of multinomial distribution.
\begin{lemma}
Consider a model $h \in \mathcal{H}$ and samples $z_i \in \mathcal{Z}$ for all $i \in [n]$. 
Then, the generalization gap can be expressed as:
\begin{align}
&\mathbb{E}_{z}[\ell(h, z)] 
- \frac{1}{n} \sum_{i=1}^n \ell(h, z_i) \notag \\
&= \sum_{j=1}^K \alpha_j(h) 
   \Big( \Pr(z \in C_j) - \frac{|n_j|}{n} \Big) \notag \\
&\quad + \sum_{j=1}^K \frac{|n_j|}{n} \Bigg[ 
      \alpha_j(h) - \frac{1}{|n_j|} 
      \sum_{i \in n_j} \ell(h, z_i) 
   \Bigg]
\label{eq:lemma8}
\end{align}
\end{lemma}

Where the expected error ($\mathbb{E}_{z}[\ell(h, z)]$) is expressed as the sum of the conditional expected error:
\begin{align*}
    &\mathbb{E}_{z}[\ell(h, z)] = \sum_{j=1}^K\mathbb{E}_{z}[\ell(h, z)|z\in C_j]\Pr(z \in C_j)\\
    &=\sum_{j=1}^K\mathbb{E}_{z_j}[\ell(h, z_j)]\Pr(z \in C_j),
\end{align*}
Here, $z_j$ is the random variable $z$ conditioned on the event $z \in C_j$. Based on this, first, we decompose the generalization gap into the following:
\begin{align}
&\mathbb{E}_{z}[\ell(h, z)] 
- \frac{1}{n} \sum_{i=1}^n \ell(h, z_i) \notag \\
&= \sum_{j=1}^K 
    \mathbb{E}_{z_j}[\ell(h, z_j)] 
    \Big( \Pr(z \in C_j) - \frac{|n_j|}{n} \Big) \notag \\
&\quad + \sum_{j=1}^K \frac{|n_j|}{n} 
    \mathbb{E}_{z_j}[\ell(h, z_j)] 
    - \frac{1}{n} \sum_{i=1}^n \ell(h, z_i)
\label{eq:16}
\end{align}

Secondly, we can expand each individual term in the second term of Equation~\eqref{eq:16} by decomposing them as follows:
\begin{align}
\sum_{j=1}^K \mathbb{E}_{z_j}[\ell(h, z_j)] \frac{|n_j|}{n} 
&= \frac{1}{n} \sum_{j=1}^K \Big(
      \mathbb{E}_{\hat{z}}[\ell(h, \hat{z})]\, |\hat{n}_j| \notag \\
&\quad + \mathbb{E}_{\widetilde{z}}[\ell(h, \widetilde{z})]\,|\widetilde{n}_j|
    \Big)
\label{eq:expansion_stable_unstable}
\end{align}

and 

\begin{align}
\frac{1}{n} \sum_{i=1}^n \ell(h, z_i) 
&= \frac{1}{n} \sum_{j=1}^K \sum_{i \in n_j} \ell(h, z_i) \notag \\
&= \frac{1}{n} \sum_{j=1}^K \Bigg(
       \sum_{i \in \hat{n}_j} \ell(h, \hat{z}_i) 
       + \sum_{i \in \widetilde{n}_j} \ell(h, \widetilde{z}_i)
    \Bigg)
\label{eq:decomposed_loss}
\end{align}

Now, the decomposed equation becomes this:
\begin{align}
&\sum_{j=1}^K \mathbb{E}_{z_j}[\ell(h, z_j)] \frac{|n_j|}{n} 
   - \frac{1}{n} \sum_{i=1}^n \ell(h, z_i) \notag \\
&= \frac{1}{n} \sum_{j=1}^K \Bigg(
       \mathbb{E}_{\hat{z}}[\ell(h, \hat{z})] |\hat{n}_j| 
       + \mathbb{E}_{\widetilde{z}}[\ell(h, \widetilde{z})] |\widetilde{n}_j| \notag \\
&\quad - \sum_{i \in \hat{n}_j} \ell(h, \hat{z}_i) 
       - \sum_{i \in \widetilde{n}_j} \ell(h, \widetilde{z}_i)
   \Bigg)
\label{eq:decomposed_difference}
\end{align}

\begin{align}
&=\frac{1}{n} \sum_{j=1}^K \Bigg(
      \big( \mathbb{E}_{\hat{z}}[\ell(h, \hat{z})] |\hat{n}_j| 
           + \mathbb{E}_{\widetilde{z}}[\ell(h, \widetilde{z})] |\widetilde{n}_j| \big) \notag \\
&\quad - \big( \sum_{i \in \hat{n}_j} \ell(h, \hat{z}_i) 
                + \sum_{i \in \widetilde{n}_j} \ell(h, \widetilde{z}_i) \big)
   \Bigg)
\label{eq:decomposed_difference2}
\end{align}

\begin{align}
&=\frac{1}{n} \sum_{j=1}^K \Bigg(
      \big( \mathbb{E}_{\hat{z}}[\ell(h, \hat{z})] |\hat{n}_j| 
           - \sum_{i \in \hat{n}_j} \ell(h, \hat{z}_i) \big) \notag \\
&\quad + \big( \mathbb{E}_{\widetilde{z}}[\ell(h, \widetilde{z})] |\widetilde{n}_j| 
                - \sum_{i \in \widetilde{n}_j} \ell(h, \widetilde{z}_i) \big)
   \Bigg)
\label{eq:decomposed_difference3}
\end{align}

Substituting these into Equation~\eqref{eq:16} yields the following:
\begin{align}
&\mathbb{E}_{z}[\ell(h, z)] - \frac{1}{n} \sum_{i=1}^n \ell(h, z_i) \notag \\
&= \sum_{j=1}^K \mathbb{E}_{z_j}[\ell(h, z_j)] 
    \Big( \Pr(z \in C_j) - \frac{|n_j|}{n} \Big) \notag \\
&\quad + \sum_{j=1}^K \frac{|\hat{n}_j|}{n} 
    \Big( \mathbb{E}_{\hat{z}}[\ell(h, \hat{z})] - \frac{1}{|\hat{n}_j|}\sum_{i \in \hat{n}_j} \ell(h, \hat{z}_i) \Big) \notag \\
&\quad + \sum_{j=1}^K \frac{|\widetilde{n}_j|}{n} 
    \Big( \mathbb{E}_{\widetilde{z}}[\ell(h, \widetilde{z})] - \frac{1}{|\widetilde{n}_j|}\sum_{i \in \widetilde{n}_j} \ell(h, \widetilde{z}_i) \Big)
\label{eq:generalization_gap_expanded}
\end{align}

\begin{align}
&\mathbb{E}_{z}[\ell(h, z)] - \frac{1}{n} \sum_{i=1}^n \ell(h, z_i) \notag \\
&= \sum_{j=1}^K \alpha_j(h) 
    \Big( \Pr(z \in C_j) - \frac{|n_j|}{n} \Big) \notag \\
&\quad + \sum_{j=1}^K \frac{|\hat{n}_j|}{n} 
    \Big( \hat{\alpha}_j(h) - \frac{1}{|\hat{n}_j|}\sum_{i \in \hat{n}_j} \ell(h, \hat{z}_i) \Big) \notag \\
&\quad + \sum_{j=1}^K \frac{|\widetilde{n}_j|}{n} 
    \Big( \widetilde{\alpha}_j(h) - \frac{1}{|\widetilde{n}_j|}\sum_{i \in \widetilde{n}_j} \ell(h, \widetilde{z}_i) \Big)
\label{eq:generalization_gap_alpha}
\end{align}

We can see that

\begin{align}
&\sum_{j=1}^{K} \frac{|\hat{n}_j|}{n} 
    \Big(\hat{\alpha}_j(h) - \frac{1}{|\hat{n}_j|}\sum_{i \in \hat{n}_j}\ell(h,\hat{z}_i)\Big) \notag \\
&= \sum_{j=1}^{K} \frac{1}{n} 
    \Big(|\hat{n}_j|\hat{\alpha}_j(h) - \sum_{i \in \hat{n}_j}\ell(h,\hat{z}_i)\Big) \notag \\
&= \frac{1}{n}\sum_{j=1}^{K}\sum_{i \in \hat{n}_j} 
    \Big(\hat{\alpha}_j(h) - \ell(h,\hat{z}_i)\Big) \notag \\
&\le \frac{1}{n}\sum_{j=1}^{K}\sum_{i \in \hat{n}_j} 
    \max_{\hat{z} \in Z_j} \Big(\ell(h,\hat{z}) - \ell(h,\hat{z}_i)\Big) \notag \\
&\le \frac{1}{n}\sum_{j=1}^{K}\sum_{i \in \hat{n}_j} \hat{\epsilon}_j(h) \notag \\
&= \sum_{j=1}^{K} \frac{|\hat{n}_j|}{n}\, \hat{\epsilon}_j(h)
\label{eqn:21}
\end{align}

and similarly 

\begin{align}
&\sum_{j=1}^{K} \frac{|\widetilde{n}_j|}{n} 
    \Big(\widetilde{\alpha}_j(h) - \frac{1}{|\widetilde{n}_j|}\sum_{i\in \widetilde{n}_j}\ell(h,\widetilde{z}_i)\Big) \notag \\
&= \sum_{j=1}^{K} \frac{1}{n} 
    \Big(|\widetilde{n}_j|\widetilde{\alpha}_j(h) - \sum_{i\in \widetilde{n}_j}\ell(h,\widetilde{z}_i)\Big) \notag \\
&= \frac{1}{n}\sum_{j=1}^{K}\sum_{i\in \widetilde{n}_j} 
    \Big(\widetilde{\alpha}_j(h) - \ell(h,\widetilde{z}_i)\Big) \notag \\
&\le \frac{1}{n}\sum_{j=1}^{K}\sum_{i\in \widetilde{n}_j} 
    \max_{\widetilde{z}\in Z_j} \Big(\ell(h,\widetilde{z}) - \ell(h,\widetilde{z}_i)\Big) \notag \\
&\le \frac{1}{n}\sum_{j=1}^{K}\sum_{i\in \widetilde{n}_j} \widetilde{\epsilon}_j(h) \notag \\
&= \sum_{j=1}^{K}\frac{|\widetilde{n}_j|}{n}\, \widetilde{\epsilon}_j(h)
\label{eqn:26}
\end{align}


Under the fixed partition of the sample space $\mathcal{Z}$ into $\{C_k\}_{j=1}^K$, each sample is assigned to a cluster according to the underlying data distribution. The corresponding cluster counts $(n_1, \dots, n_K)$ therefore follow a multinomial distribution with parameters $n$ and $(\Pr(C_1), \dots, \Pr(C_K))$. For each cluster $C_j$, let the total number of samples be denoted by $n_j = \hat{n}_j + \widetilde{n}_j$, where $\hat{n}_j$ and $\widetilde{n}_j$ represent the counts of stable and unstable samples within $C_j$, respectively.

Hence, according to Theorem 3 in \cite{kawaguchi2022robustness}, for any $\delta > 0$, with probability at least $1-\delta$, the following holds:
\begin{equation}
\begin{split}
\sum_{j=1}^K \alpha_j(h) 
&\Big( \Pr(z \in C_j) - \frac{|n_j|}{n} \Big) \le u_{3}(K, \mathcal{S}, \delta)
\label{eqn:27}
\end{split}
\end{equation}
Where 
\begin{align}
u_{3}(K, \mathcal{S}, \delta) = \mathcal{Q}_1 \sqrt{\frac{\ln(2K/\delta)}{n}} 
+ \frac{2 \mathcal{Q}_2 \ln(2K/\delta)}{n} \nonumber
\end{align}
\begin{align*}
\mathcal{Q}_1 &:= \sum_{k \in T_S} \left( \alpha_{\mathcal{T}_{S}^c}(A_S) + \sqrt{2} \, \alpha_k(A_S) \right) \sqrt{\frac{|\mathcal{I}_k^S|}{n}}, \\
\mathcal{Q}_2 &:= \alpha_{\mathcal{T}_{S}^c}(A_S) \cdot |T_S| + \sum_{k \in T_S} \alpha_k(A_S),
\end{align*}

\begin{align*}
T_S &:= \left\{ k \in [K] : |\mathcal{I}_k^S| \geq 1 \right\}, \\
\mathcal{I}_k^S &:= \left\{ i \in [n] : z_i \in C_k \right\}, \\
\alpha_k(h) &:= \mathbb{E}_{z} \left[ \ell(h, z) \mid z \in C_k \right], \\
\alpha_{\mathcal{T}_{S}^c}(A_S) &:= \max_{k \in \mathcal{T}_{S}^c} \alpha_k(A_S), \\
\mathcal{T}_{S}^c &:= [K] \setminus T_S.
\end{align*}

Combining \eqref{eq:lemma8}, \eqref{eqn:21}, \eqref{eqn:26}, and \eqref{eqn:27} 
completes the proof of Theorem \ref{thm:cluster_bound} in bound \eqref{eqn:cluster_bound_both}.
\end{proof}

\begin{proof}[Proof of Lemma~\ref{lem:lemma1}]
Recall that the input space $\mathcal{Z}$ can be decomposed into $\mathcal{Z}_1 = \hat{\mathcal{Z}}$ (stable subset) and $\mathcal{Z}_2 = \widetilde{\mathcal{Z}}$ (unstable subset), with $\alpha_i(\mathcal{A}_S) := \mathbb{E}_{z \sim \mathcal{Z}_i}[\ell(\mathcal{A}_S, z)]$ denoting the expected loss on subset $\mathcal{Z}_i$.

By the law of total expectation, the expected risk decomposes as
\begin{align}
\mathcal{R}_{\mathcal{Z}}(\mathcal{A}_S) & = \mathbb{E}_{z \sim \mathcal{Z}}[\ell(\mathcal{A}_S, z)] \notag \\
&= \mathbb{E}_{z \sim \mathcal{Z}_1}[\ell(\mathcal{A}_S, z)]\mathbb{P}(z \in \mathcal{Z}_1) \notag \\
& ~~~+ \mathbb{E}_{z \sim \mathcal{Z}_2}[\ell(\mathcal{A}_S, z)]\mathbb{P}(z \in \mathcal{Z}_2) \notag \\
& = \alpha_1(\mathcal{A}_S)\mathbb{P}(z \in \mathcal{Z}_1) + \alpha_2(\mathcal{A}_S)\mathbb{P}(z \in \mathcal{Z}_2).
\end{align}

Since $\hat{\mathcal{Z}}$ is constructed as the stable subset where $\mathcal{A}_S$ predicts correctly with high probability, we have $\alpha_1(\mathcal{A}_S) \approx 0$ under $0$-$1$ loss. Moreover, $\mathbb{P}(z \in \mathcal{Z}_2) \leq 1$, so
\begin{align}
\mathcal{R}_{\mathcal{Z}}(\mathcal{A}_S) & \leq \mathbb{P}(z \in \mathcal{Z}_2)\alpha_2(\mathcal{A}_S) \notag\\
& \leq \alpha_2(\mathcal{A}_S).
\end{align}
This completes the proof of Lemma \ref{lem:lemma1}.
\end{proof}

\begin{proof}[Proof of Corollary \ref{thm:ours}]
In the case of $0$-$1$ loss, or when the model exhibits smooth loss variations, the cumulative contribution of the stable clusters to the robustness term becomes negligible. Formally, this can be expressed as:
\[
\sum_{j=1}^{K} \frac{|\hat{n}_j|}{n}\, \hat{\epsilon}_j(\mathcal{A_S}) \approx 0.
\]

Consequently, the overall generalization bound in Theorem \ref{thm:cluster_bound} simplifies to include only the contributions from the unstable clusters. By combining equations \eqref{eq:lemma8}, \eqref{eqn:26}, and \eqref{eqn:27}, we obtain the desired bound in Corollary 1:


\begin{align}
\mathcal{R}_Z(\mathcal{A_S})
\leq\;
& \hat{\mathcal{R}}_{\mathcal{S}}(\mathcal{A_S})
+ \sum_{j \in \widetilde{T}_S}
\frac{\widetilde{n}_j(\mathcal{A_S})}{n}
\cdot
\widetilde{\epsilon}_j(\mathcal{A_S})
\notag\\
& + u_{3}(K, \mathcal{S}, \delta)
\end{align}
This completes the proof of Corollary \ref{thm:ours}. 
\end{proof}

\section{\textbf{Additional results: sensitivity analysis of \texorpdfstring{$\kappa$}{kappa}}}\label{Appendix_B}
Table~\ref{tab:sensitivity_kappa} demonstrates the robustness of our stability analysis across a broad range of neighborhood sizes $\kappa \in {20, 30, 80, 120, 200}$. The proportion of unstable samples remains highly consistent for each model, with maximum variations below $0.2\%$ across all values of $\kappa$. For example, ResNet18 V1 exhibits a narrow range of $39.51\%$–$39.54\%$, while Swin\_V2\_b varies only between $11.91\%$ and $11.93\%$.

This negligible sensitivity to $\kappa$ indicates that the proposed stable/unstable partitioning is inherently robust to the choice of neighbors value. Consequently, the use of $\kappa = 20$ for neighborhood generation is well-justified and does not compromise the reliability of the analysis.
\begin{table}
    \centering
    \caption{Percentage of unstable samples (\%) across $\kappa \in \{20, 30, 80, 120, 200\}$ for various deep neural networks (DNNs) pretrained on the ImageNet dataset. The selected $\kappa$ values span strict, practical, and asymptotic regimes, providing a comprehensive sensitivity analysis across neighborhood scales.}
    \label{tab:sensitivity_kappa}
    \resizebox{\columnwidth}{!}{%
    \setlength{\tabcolsep}{4pt}
    \renewcommand{\arraystretch}{1.1}
    \begin{tabular}{lccccc}
        \toprule
        \textbf{Model} & $\boldsymbol{\kappa=20}$ & $\boldsymbol{\kappa=30}$ & $\boldsymbol{\kappa=80}$ & $\boldsymbol{\kappa=120}$ & $\boldsymbol{\kappa=200}$ \\
        \midrule
        ResNet18 V1    & 39.53 & 39.51 & 39.52 & 39.53 & 39.54 \\
        ResNet34 V1    & 30.19 & 30.19 & 30.17 & 30.18 & 30.17 \\
        ResNet50 V1    & 28.79 & 28.78 & 28.77 & 28.79 & 28.80 \\
        ResNet101 V1   & 26.32 & 26.33 & 26.34 & 26.31 & 26.33 \\
        ResNet152 V1   & 24.64 & 24.64 & 24.62 & 24.64 & 24.63 \\
        ResNet50 V2    & 17.14 & 17.16 & 17.15 & 17.16 & 17.16 \\
        ResNet101 V2   & 15.57 & 15.55 & 15.55 & 15.57 & 15.56 \\
        ResNet152 V2   & 10.98 & 10.98 & 10.96 & 10.96 & 10.95 \\
        VGG13          & 42.32 & 42.30 & 42.29 & 42.28 & 41.96 \\
        VGG13 BN       & 42.29 & 42.28 & 42.27 & 42.26 & 42.26 \\
        VGG19          & 36.47 & 36.46 & 36.45 & 36.44 & 36.46 \\
        VGG19 BN       & 33.43 & 33.42 & 33.43 & 33.41 & 33.41 \\
        DenseNet121    & 27.63 & 27.63 & 27.62 & 27.61 & 27.63 \\
        DenseNet161    & 20.68 & 20.66 & 20.66 & 20.68 & 20.67 \\
        DenseNet169    & 23.69 & 23.68 & 23.69 & 23.67 & 23.67 \\
        DenseNet201    & 21.08 & 20.84 & 20.83 & 20.83 & 20.83 \\
        Swin\_b        & 12.65 & 12.65 & 12.65 & 12.62 & 12.63 \\
        Swin\_t        & 15.58 & 15.58 & 15.58 & 15.57 & 15.56 \\
        Swin\_V2\_b    & 11.93 & 11.93 & 11.91 & 11.92 & 11.91 \\
        Swin\_V2\_t    & 16.62 & 16.62 & 16.63 & 16.61 & 16.62 \\
        \bottomrule
    \end{tabular}}
\end{table}

\bibliographystyle{cas-model2-names}

\bibliography{cas-refs}


\vskip3pc

\end{document}